\documentclass{IEEEtran}
\IEEEoverridecommandlockouts
\pdfoutput=1
\usepackage{cite}
\usepackage{amsmath,amssymb,amsfonts}
\usepackage{amsthm}
\usepackage{algorithmic}
\usepackage{graphicx}
\usepackage{textcomp}
\usepackage{wrapfig}
\usepackage{epstopdf}
\usepackage{array}
\usepackage{subfigure}
\usepackage{bm}
\usepackage{picins}
\newtheorem{Theo}{Theorem}
\usepackage{amsthm}

\usepackage{cuted}\stripsep -3pt plus 3pt minus 2pt
\ifCLASSINFOpdf \else \fi

\usepackage{algorithm}
\usepackage{slashbox}

\theoremstyle{definition}

\usepackage{color}
\usepackage{slashbox}
\usepackage[normalem]{ulem}
\IEEEoverridecommandlockouts
\usepackage{mathrsfs}
\usepackage{amssymb}
\usepackage{amsmath}
\usepackage{amssymb}
\usepackage{amsthm}
\usepackage{multirow,makecell}
\usepackage{makecell}
\usepackage[table]{xcolor}
\usepackage{subfigure}
\usepackage{setspace}
\usepackage{stfloats}
\usepackage{graphicx}
\usepackage{epstopdf}
\usepackage{soul}
\soulregister\cite7
\soulregister\ref7
\usepackage{url}

\ifCLASSINFOpdf \else \fi

\usepackage{algorithm}
\usepackage{algorithmic}

\usepackage{cite}

\begin{document}
	\title{Three-Dimensional Spectrum Occupancy Measurement using UAV: Performance Analysis and Algorithm Design}
	\author{Zhiqing Wei,
		Rubing Yao,
		Jie Kang,
		Xu Chen,
		and Huici Wu
		\thanks{Zhiqing Wei, Rubing Yao and Xu Chen are with Key Laboratory of Universal Wireless Communications, Ministry of Education, School of Information and Communication Engineering, Beijing University of Posts and Telecommunications (BUPT), Beijing, 100876, China (e-mail: weizhiqing@bupt.edu.cn, yaoru0422@163.com, chenxu96330@bupt.edu.cn).}
		\thanks{Jie Kang is with China Unicom, Beijing, 100031, China (e-mail: kangj56@chinaunicom.cn).}
		\thanks{Huici Wu is with the National Engineering Lab for Mobile Network Technologies, Beijing University of Posts and Telecommunications (BUPT), Beijing, 100876, China, and also with Peng Cheng Laboratory, Shenzhen, China. (e-mail: dailywu@bupt.edu.cn).}}
\maketitle	
\begin{abstract}
	   Spectrum sharing, as an approach to significantly improve spectrum efficiency in the era of 6th generation mobile networks (6G), has attracted extensive attention. Radio Environment Map (REM) based low-complexity spectrum sharing is widely studied where the spectrum occupancy measurement (SOM) is vital to construct REM. The SOM in three-dimensional (3D) space is becoming increasingly essential to support the spectrum sharing with space-air-ground integrated network being a great momentum of 6G. In this paper, we analyze the performance of 3D SOM to further study the tradeoff between accuracy and efficiency in 3D SOM. We discover that the error of 3D SOM is related with the area of the boundary surfaces of licensed networks, the number of discretized cubes, and the length of the edge of 3D space. Moreover, we design a fast and accurate 3D SOM algorithm that utilizes unmanned aerial vehicle (UAV) to measure the spectrum occupancy considering the path planning of UAV, which improves the measurement efficiency by requiring less measurement time and flight time of the UAV for satisfactory performance. The theoretical results obtained in this paper reveal the  essential dependencies that describe the 3D SOM methodology, and the proposed algorithm is beneficial to improve the efficiency of 3D SOM. It is noted that the theoretical results and algorithm in this paper may provide a guideline for more areas such as spectrum monitoring, spectrum measurement, network measurement, planning, etc. 
\end{abstract}
					

		\section{Introduction}
		\label{sec:introduction}
		\subsection{Background and Motivations}
		
		\IEEEPARstart{W}{ith} the launch of the world's first 6G White Paper in 2019 \cite{ref1}, the research on 6G networks has been initiated. 6G will bring a series of new applications such as Internet of everything (IoE), holographic telepresence, extended reality (XR) and so on\cite{ref2}, which will profoundly change the human society of 2030s and beyond. In order to satisfy the requirements of these application scenarios, 6G should have larger system capacity, higher data rate, higher spectrum efficiency and so on. However, spectrum resources are increasingly scarce with the increasing demand. The International Telecommunication Union (ITU) has warned that the growing use of mobile broadband will lead to global spectrum congestion \cite{ref3}. Spectrum sharing technology enables unauthorized users to utilize the spectrum that is not fully utilized by licensed users in time and space dimensions, which will significantly improve the spectrum efficiency \cite{ref3}.
		
		To reduce the resource scheduling complexity of spectrum sharing, the Radio Environment Map (REM) based spectrum sharing was proposed and widely studied \cite{ref4}. The REM was introduced to cooperatively collect, store and share information among the network users regarding the spectrum occupancy, location information of licensed users, spectrum management rules \cite{ref5} \cite{ref6}, interference management methods\cite{ref7} and so on, among which the spectrum occupancy information is most critical for the spectrum sharing. With the space-air-ground integrated network being one of the key features of 6G \cite{ref8}, spectrum occupancy measurement (SOM) in three-dimensional (3D) space is becoming increasingly significant to provide valuable information to regulators about the efficiency of the current use of the spectrum resources \cite{ref9}. However, the SOM faces several challenges in the era of 6G. Firstly, base station deployment schemes are diverse due to different scenario requirements for the network devices related to their specific functionality, causing complex 3D spectrum environments. Thus, it is difficult to perform 3D SOM because of nonlinearity and nonstationary of 3D signal propagation models. Secondly, there exists complex tradeoff between accuracy and efficiency in the 3D SOM. Hence, we need to design the 3D SOM scheme with high accuracy and efficiency.
		
		\subsection{Related Works}
		The statistical models of spectrum occupancy was summarized in \cite{ref10}. The Markov chain is a very natural choice for statistical modeling of the spectrum occupancy, as the spectrum occupancy rate is either 0 or 1 after energy detection and the occupancy status changes between these two cases \cite{ref11}. Indeed, the Markov chain provides an accurate description of the time dimension of the spectrum occupancy while seldom considering the frequency dimension and the space dimension. In \cite{ref12}, Lopez-Benitez \textit{et al.} proposed a space dimension model of the spectrum occupancy for the duty cycle as a function of various parameters, such as the probability of false alarm, the activity factor, and so on. These models depict the statistical behaviors of the spectrum occupancy. However, they did not provide the measurement of the spectrum occupancy. The accurate spectrum measurement is of great value in practical applications, which is the theme of this paper.
		
		The results from the SOM around the world showed significant spectrum opportunities or low spectrum occupancy in frequencies above 1 GHz \cite{ref13}. There are some literatures that attempt to deal with the challenges of 3D SOM in Section I-A. As to the 3D SOM, a method of cooperative SOM was designed in \cite{ref14}, which could improve the reliability of SOM by integrating measurement results from several locations. A combined approach employing low complexity array processing and conventional energy detection toward spectrum occupancy was proposed in \cite{ref15}. In \cite{ref16}, Ivanov \textit{et al.} applied experimental equipment to measure spectrum occupancy in dense indoor environment. In \cite{ref17}, Ayg\"ul \textit{et al.} proposed a 3D spectrum occupancy prediction method based on composite two-dimensional (2D) long-term short-term (LTST) memory model, which revealed a high detection performance with low complexity and high robustness. Recently, unmanned aerial vehicles (UAVs) are widely applied in both civil and military fields, which are promising in 3D SOM \cite{ref18}. Firstly, the UAV has the advantages of high flexibility and maneuverability. Thus, it is possible to apply UAV in 3D SOM with high degrees of freedom \cite{ref19}. In addition, most of the measurements are performed at high-altitude outdoor so as to get the accurate estimation of the licensed user's activity in various spectrum bands. Therefore, the high-altitude UAV communication platforms having highly reliable Line of Sight (LoS) transmission \cite{ref20} is applied in 3D SOM. In \cite{ref21}, Al-Hourani installed a software-defined receiver on UAV for SOM. To achieve the tradeoff between accuracy and efficiency in SOM, Faint \textit{et al.} \cite{ref22} studied this tradeoff relation, and showed that increasing the number of measurements can improve the accuracy. However, when the number of measurements is large enough, the improvement is not significant. In \cite{ref23}, we revealed the tradeoff relation between the accuracy and the number of measurements in 2D SOM. However, the relation depends on the specific distribution of licensed networks.
		
		All the above works are limited in the following aspects. Firstly, the existing literatures focused on the measurement accuracy and complexity. However, the efficiency of measurement has not attracted wide attention. Secondly, the tradeoff between accuracy and efficiency in 2D SOM was studied. However, the tradeoff in 3D SOM has attracted limited attention. Finally, the 3D SOM algorithms are rarely studied. 
		\subsection{Our Contributions}
		Considering these limitations, the performance of 3D SOM is analyzed in this paper. Besides, a fast and accurate SOM algorithm using UAV is designed. The main contributions of this paper are as follows.
		
		\begin{itemize}
			\item[1.] We obtain the relation between the error of 3D SOM $P_e$ and the number of measurements $M$, i.e.,  ${P_e} = \Theta (\frac{1}{{\sqrt[3]{M}}})$, which is a decreasing and convex function related with the area of the boundary surfaces of licensed networks.
			
			\item[2.] A path planning algorithm of UAV is designed to implement the 3D SOM. This algorithm reduces the number of measurements, shortens the flight distance of UAV, and reduces the energy consumption in UAV-assisted SOM compared with regular measurement algorithm, which is verified theoretically and numerically.
		\end{itemize}
	
		\begin{table}[]
			\caption{\label{tab1}Key parameters and abbreviations}
			\begin{center}
				\begin{tabular}{m{2cm}<{\centering}|m{6cm}<{\centering}}
					\hline
					\hline
					{\textbf{Symbol}} & {\textbf{Description}} \\
					\hline
					SOM & Spectrum Occupancy Measurement\\
					\hline
					REM & Radio Environment Map\\
					\hline
					ACO & Ant Colony Optimization\\
					\hline
					$I_i$ & Radio parameter of cube \#$i$\\
					\hline
					$P_e$ & Radio parameter error (RPE)\\
					\hline
					$P_{e,i}$ & Radio parameter error  of cube \#$i$\\
					\hline
					$M$ & Number of cubes\\
					\hline
					$K$ & Number of cubes with impure radio environment\\
					\hline
					$S$ & Area of all the boundary surfaces of licensed networks\\
					\hline
					$S_i$ & Area of the boundary surface of licensed network in cube \#$i$\\
					\hline
					$L$ & Length of the 3D space's edge\\
					\hline
					$\varepsilon$ & Side length of a cube\\
					\hline
					$N$ & Number of radio parameters\\
					\hline
					$T$ & Number of licensed networks\\
					\hline
					$d$ & Measurement interval\\
					\hline
					$r$ & Number of iterations\\
					\hline
					$D_r$ & Number of measurements \\
					\hline
					\hline
				\end{tabular}
			\end{center}
		\end{table}
		\begin{figure}[t]
	    \centering
	    \includegraphics[width=0.35\textheight]{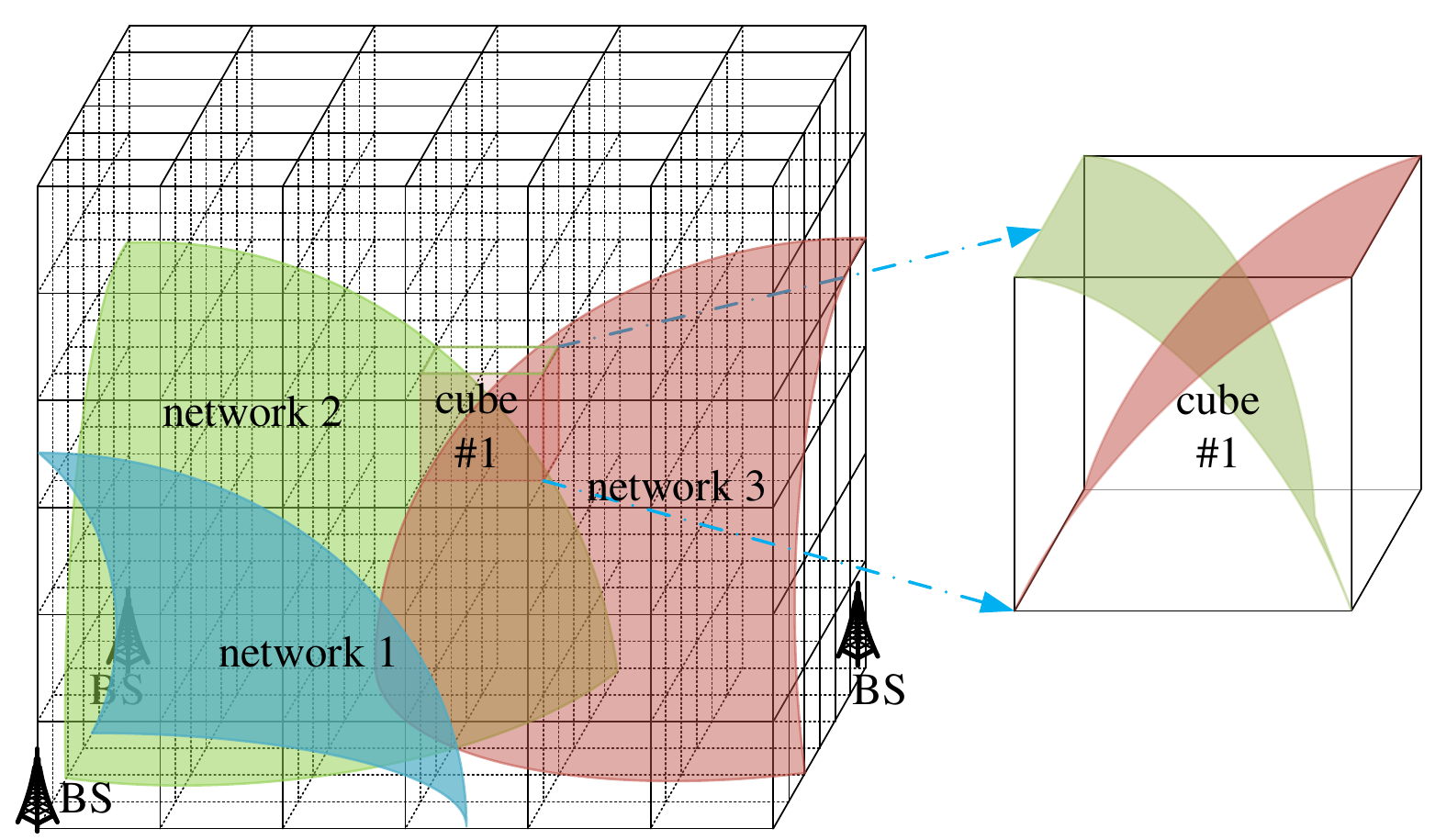}
 	    \DeclareGraphicsExtensions.
	    \caption{The spectrum occupancy information of 3D.}
	    \label{fig 1}
        \end{figure}
		\subsection{Outline of This Paper}
		The remaining parts of this paper are organized as follows. The representation of radio parameters is introduced in Section II. The performance of 3D SOM is derived in Section III. In Section IV, the path planning algorithm of UAV is designed for 3D SOM. Section V presents the numerical results to verify the theoretical results. Finally, Section VI concludes this paper. The key parameters and abbreviations are listed in Table \ref{tab1} .
		\begin{figure*}[!t]
			\centering
			\includegraphics[width=0.70\textheight]{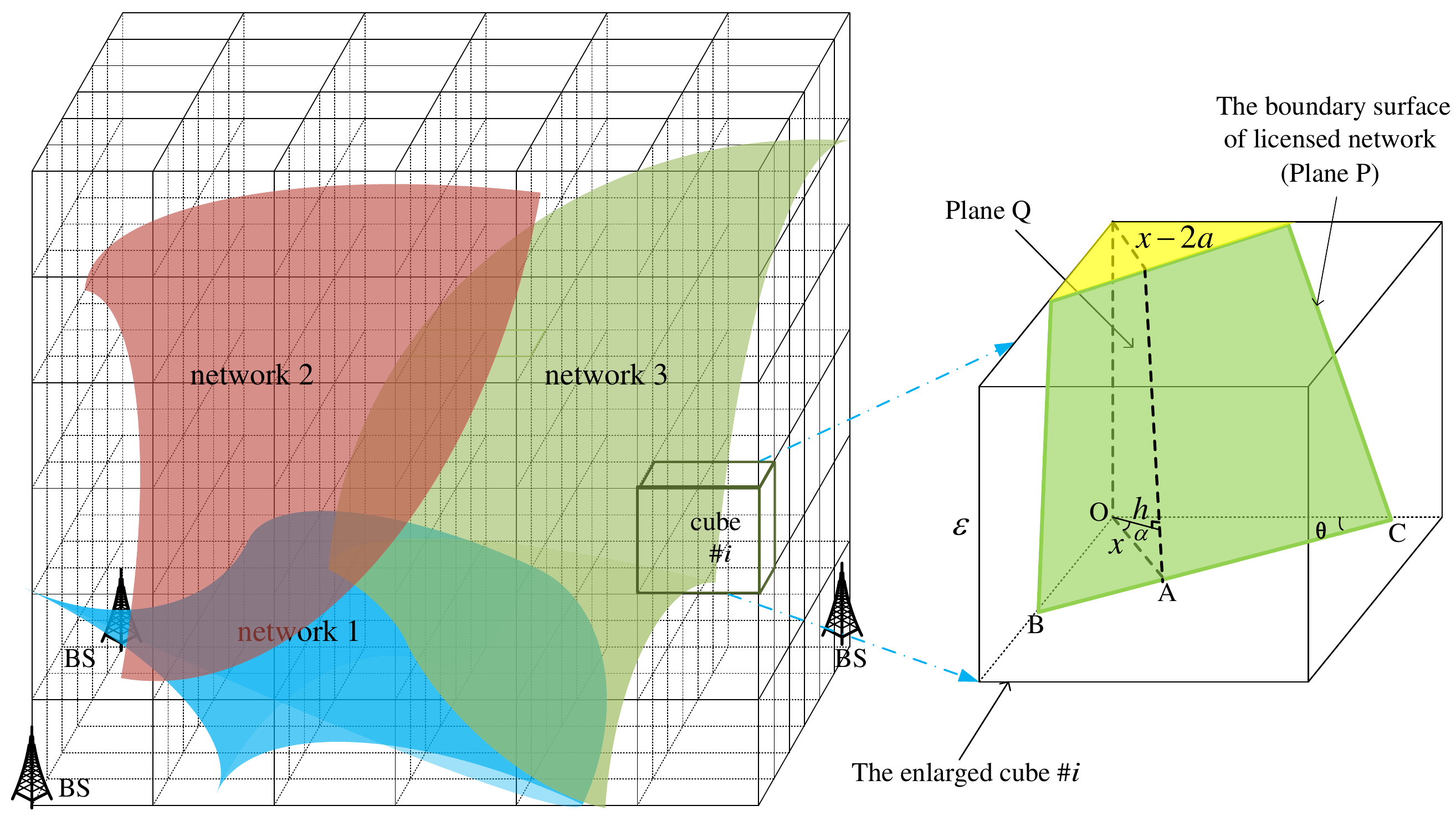}
			\DeclareGraphicsExtensions.
			\caption{The 3D spectrum occupancy.}
			\label{fig 2}
		\end{figure*}	
		
		\section{Representation of Radio Parameters}

		The high-altitude UAV platforms have the highly reliable Line of Sight (LoS) transmission. The ground sensors are traditionally utilized when measuring the spectrum occupancy on the ground, which is not the focus of this paper. Therefore, the role of shadowing or fast fading fluctuations is not taken into account since it is small in the air. As shown in Fig. \ref{fig 1}, there are several licensed networks in a specific space, and the spectrum occupancy of licensed networks needs to be measured. The entire 3D space is divided into cubes, which is the smallest unit of SOM.
		
		The definition of radio parameter follows the previous work \cite{ref23}. Whether a licensed network exists at a certain location is recognized by measuring the signal strength. The binary representation of licensed network $k$ at the location with coordinates $\left(x,y,z\right)$ is
		\begin{equation}
			R(k,x,y,z) = \left\{ \begin{array}{l}
				1\;{\rm if\;network\;}k\;{\rm is\;detected\; at\;}(x,y,z)\\
				0\;{\rm otherwise}
			\end{array} \right.,
			\label{equ 1}
		\end{equation}
		The spectrum occupancy information at a location is defined by the \textit{radio parameter}, which is defined by the sum of the binary representations $R\left(k,x,y,z\right),\forall k$ as follows.
		\begin{equation}
			I\left(x,y,z\right)=\sum_{k=1}^{T}{R\left(k,x,y,z\right)\times2^{k-1}},
			\label{equ 2}
		\end{equation}
		where $T$ is the number of licensed networks, and $N=2^T$ is the number of radio parameters. In the cube \#1 of Fig. \ref{fig 1}, there are 4 radio parameters. The radio parameter of cube \#$i$ is defined as follows \cite{ref23}.
		\begin{equation}
			{I_i} = \mathop {{\mathop{\rm argmax}\nolimits} }\limits_j {p_{ij}},
			\label{equ 3}
		\end{equation}
		where $p_{ij}$ is the fraction of the volume in cube \#$i$ with radio parameter $j$. Without lose of generality, it is assumed that the center of cube \#$i$ undergoes different radio parameters with a uniform distribution.
		
		The radio parameter error (RPE) of cube \#$i$ is thus defined as \cite{ref23}
		\begin{equation}
			{P_{e,i}} = 1 - \mathop {max}\limits_j {p_{ij}},
			\label{equ 4}
		\end{equation}
		and the RPE of the entire 3D space is defined as \cite{ref23}
		\begin{equation}
			P_e=\sum_{i=1}^{M}{\alpha_iP_{e,i}},
			\label{equ 5}
		\end{equation}
		where $M$ is the number of cubes, and $\alpha_i$ is the fraction of the volume of cube \#$i$ in the entire 3D space. If the entire 3D space is evenly divided into small cubes, we have $\alpha_i=\frac{1}{M}$.
		
		The RPE defined above represents the error of the SOM in the entire 3D space.
		
		\section{Performance Analysis of 3D SOM}
		
		In order to explore the factors influencing the accuracy of 3D SOM, we investigate the relation between RPE and the number of cubes in this section. Firstly, we reveal the result considering the irregular shape of licensed networks' coverage in \textbf{Theorem 1}. Then, assuming the parameters related with the shape of licensed networks' coverage follow uniform distribution, we reveal a special case in \textbf{Theorem 2}.
		\begin{Theo}
			If the shape of licensed networks' coverage is irregular, the RPE is as follows.
			\begin{equation}
				{P_e} = \Theta (\frac{1}{{\sqrt[3]{M}}}),
				\label{equ 6}
			\end{equation}
			where $\Theta (*)$ is infinitesimal of the same order.
		\end{Theo}
	\begin{figure}[t]
		\centering
		\includegraphics[width=0.35\textheight]{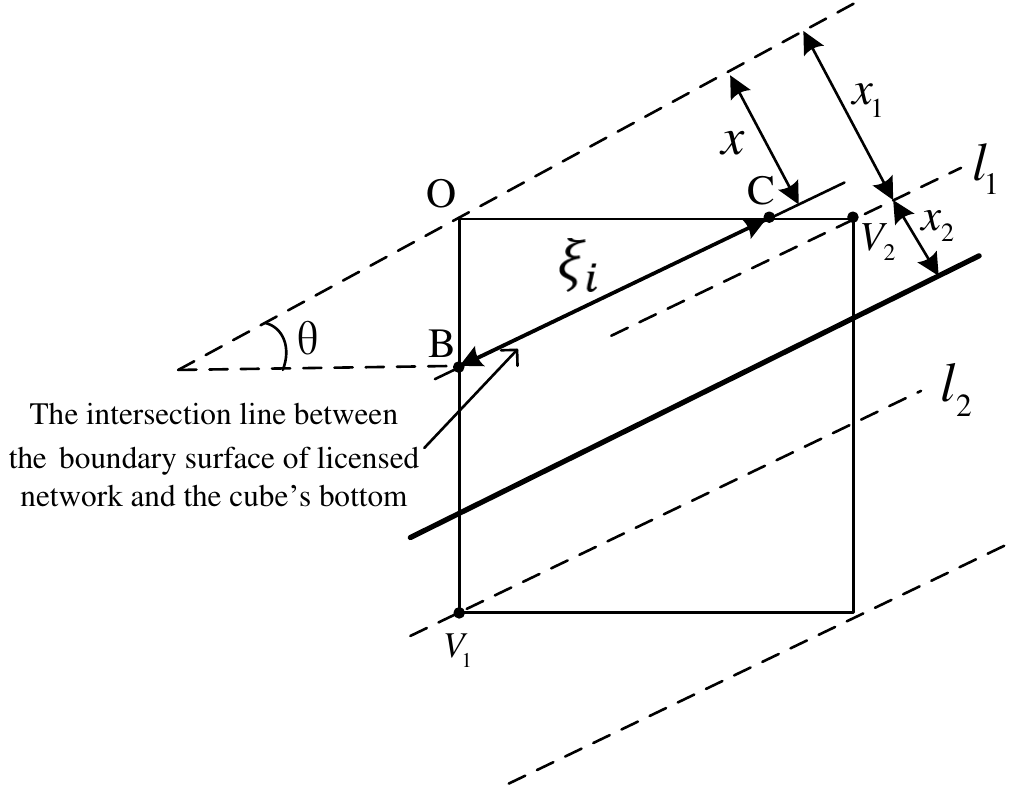}
		\DeclareGraphicsExtensions.
		\caption{The bottom of cube $\#i$.}
		\label{fig 3}
	\end{figure}
		\begin{proof}
			When a cube consists of the spots with different radio parameters, namely, a cube is cut by the licensed networks' boundaries and contains more than one radio parameter, it is regarded as a cube with impure radio environment. Furthermore, the number of the cubes with impure radio environment is defined as $K$. Fig. \ref{fig 2} shows the case that the licensed networks' boundaries cut the cubes. We investigate the boundary surface of a licensed network within a cube. When the cube is small, this boundary surface can be approximated by a plane, as shown in Fig. \ref{fig 2}, where the plane in the cube is the shaded area denoted as plane P. The RPE of the cube in Fig. \ref{fig 2} is the fraction of the volume of the shaded space in cube $\#i$.
			
		    Fig. \ref{fig 3} is a bottom of in the enlarged cube $\#i$ of Fig. \ref{fig 2}. The parameters $x$ and $\theta$ determine a line in Fig. \ref{fig 3}. As shown in Fig. \ref{fig 2} and Fig. \ref{fig 3}, $X = x$ is the distance between vertex O and the intersection line between the boundary surface of licensed network and the cube's bottom (denoted as line BC), $\Theta = \theta$ is the angle between this intersection line and horizontal line, and ${\rm A} = \alpha$ is the angle between the boundary surface of licensed network and the bottom of the cube. As the shape of licensed network's coverage could   be irregular, the probability density functions (PDFs) of the random variables $X = x$, $\Theta = \theta$ and ${\rm A} = \alpha$ are general, namely, we do not make any assumption on the PDFs of these random variables. Thus, these random variables are independent and the PDFs are generalized as follows \cite{ref23}.
			
			\begin{equation}
				{f_X}(x),0 \le x \le \frac{{\sqrt 2 }}{2}\varepsilon \sin \left( {\theta  + \frac{\pi }{4}} \right),
				\label{equ 7}
			\end{equation}
			\begin{equation}
				f_\Theta\left(\theta\right),0 \le x \le \frac{\pi }{4},
				\label{equ 8}
			\end{equation}
			\begin{equation}
				{f_A}(\alpha ),0 \le x \le \frac{\pi }{4},
				\label{equ 9}
			\end{equation}
			where $\varepsilon$ is the side length of a cube.
			
			The length of the intersection line BC in Fig. \ref{fig 3} is
			\begin{equation}
				{\xi _i} = \left\{ \begin{array}{l}
					x\left( {\tan \theta  + \cot \theta } \right),x \le {x_1}\\
					\frac{\varepsilon }{{\cos \theta }},{x_1} \le x \le {x_2}
				\end{array} \right.,
				\label{equ 10}
			\end{equation}
			where $x_1$ and $x_2$ are shown in Fig. \ref{fig 3}, with values
			\begin{equation}
				{x_1} = \varepsilon \sin \theta,
				\label{equ 11}
			\end{equation}
			\begin{equation}
				{x_2} = {\left[ {\frac{{\sqrt 2 }}{2}\varepsilon \sin \left( {\theta  + \frac{\pi }{4}} \right) - \varepsilon \sin \theta } \right]^ + }.
				\label{equ 12}
			\end{equation}
			
			$x_2$ is half the distance between line $l_1$ and line $l_2$, where $l_1$ and $l_2$ are the lines passing through the vertices $V_1$ and $V_2$ paralleling to the intersection line BC, and ${\left[  *  \right]^ + } = \max \{ 0, * \}$.
			
			Fig. \ref{fig 4} is viewed from the perspective perpendicular to plane Q, shown as the plane surrounded by four dotted lines in the enlarged cube $\#i$ of Fig. \ref{fig 2}. The difference value between the top and bottom of cube $\#i$ in the direction of line OA is $2a=\varepsilon\tan{\alpha}$. The distance between vertex O and the boundary surface of licensed network is $h=x\cos{\alpha}$, and the maximum of $x$ due to the symmetry is $x_1+x_2+a$.
			
			The shape of the boundary surface of licensed network in cube $\#i$ changes with the value of $x$. According to the relation between $a$ and $x_2$, the shape has the following cases.
			
			\subsection{$a < {x_2}, i.e., \tan \alpha  < \cos \theta  - \sin \theta$}
			
			In this case, as the value of $x$ increases, the shape of the boundary surface of licensed network in cube $\#i$ becomes triangle, trapezoid, pentagon or parallelogram in turn. We define two intermediate variables as $A = \frac{{{x_1} - \left( {x - 2a} \right)}}{{\sin \alpha }},B = \frac{{x - \left( {{x_1} + 2{x_2}} \right)}}{{\sin \alpha }}$.
			
			According to the relation between $a$ and $x_1$, the area of the boundary surface of licensed network in cube $\#i$ is calculated as follows.
			
			\begin{itemize}
				\item[$\bullet$]When $2a < {x_1}, i.e., \tan \alpha  < \sin \theta$,
			\end{itemize}		
			\begin{equation}
				{S_i}\left( x \right) = \left\{ \begin{array}{l}
					\frac{{{x^2}(\tan\theta  + \cot\theta )}}{{2\sin \alpha }},x \le 2a\\
					\frac{{(x - a)(\tan\theta  + \cot\theta )\varepsilon }}{{\cos \alpha }},2a < x \le {x_1}\\
					\frac{{(x - 2a)(\tan\theta  + \cot\theta ) + \frac{\varepsilon }{{\cos \theta }}}}{2}A\\
					+ \frac{\varepsilon }{{\cos \theta }}(\frac{\varepsilon }{{\cos \alpha }} - A),{x_1} < x \le {x_1} + 2a\\
					\frac{\varepsilon }{{\cos \theta }}\frac{\varepsilon }{{\cos \alpha }},{x_1} + 2a < x \le {x_1} + {x_2} + a
				\end{array} \right..
				\label{equ 13}
			\end{equation}
			
			\begin{itemize}
				\item[$\bullet$]When $2a > {x_1}, i.e., \tan \alpha  > \sin \theta$,
			\end{itemize}	
			\begin{equation}
				{S_i}\left( x \right) = \left\{ \begin{array}{l}
					\frac{{{x^2}(\tan\theta  + \cot\theta )}}{{2\sin \alpha }},x \le {x_1}\\
					\frac{{{x_1}^2(\tan\theta  + \cot\theta )}}{{2\sin \alpha }} + \frac{\varepsilon }{{\cos \theta }}\frac{{x - {x_1}}}{{\sin \alpha }},{x_1} < x \le 2a\\
					\frac{{(x - 2a)(\tan\theta  + \cot\theta ) + \frac{\varepsilon }{{\cos \theta }}}}{2}A\\
					+ \frac{\varepsilon }{{\cos \theta }}(\frac{\varepsilon }{{\cos \alpha }} - A),2a < x \le {x_1} + 2a\\
					\frac{\varepsilon }{{\cos \theta }}\frac{\varepsilon }{{\cos \alpha }},{x_1} + 2a < x \le {x_1} + {x_2} + a
				\end{array} \right..	
				\label{equ 14}
			\end{equation}
			
			\subsection{$a > {x_2}, i.e., \tan \alpha  > \cos \theta  - \sin \theta$}
			
			In this case, as the value of $x$ increases, the shape of the boundary surface of licensed network in cube $\#i$ becomes triangle, trapezoid or pentagon in turn. Similarly, we define an intermediate variable as $C=\tan\theta+\cot\theta$.
			
			According to the relation between a and $x_1$, the area of the boundary surface of licensed network in cube $\#i$ is calculated as follows.
			
			\begin{itemize}
				\item[$\bullet$]When $2a < {x_1}, i.e., \tan \alpha  < \sin \theta$,
			\end{itemize}	
			\begin{equation}
				{S_i}\left( x \right) = \left\{ \begin{array}{l}
					\frac{{{x^2}(\tan\theta  + \cot\theta )}}{{2\sin \alpha }},x \le 2a\\
					\frac{{(x - a)(\tan\theta  + \cot\theta )\varepsilon }}{{\cos \alpha }},2a < x \le {x_1}\\
					\frac{{((x - 2a)C + \frac{\varepsilon }{{\cos \theta }})}}{2}A + \frac{{(xC + \frac{\varepsilon }{{\cos \theta }})}}{2}B \\
					+ \frac{\varepsilon }{{\cos \theta }}(\frac{\varepsilon }{{\cos \alpha }} - A - B),{x_1} < x \le {x_1} + {x_2} + a
				\end{array} \right..	
				\label{equ 15}
			\end{equation}
			
			\begin{itemize}
				\item[$\bullet$]When $2a > {x_1}, i.e., \tan \alpha  > \sin \theta$,
			\end{itemize}	
			\begin{equation}
				{S_i}\left( x \right) = \left\{ \begin{array}{l}
					\frac{{{x^2}(\tan\theta  + \cot\theta )}}{{2\sin \alpha }},x \le 2a\\
					\frac{{{x_1}^2(\tan\theta  + \cot\theta )}}{{2\sin \alpha }} + \frac{\varepsilon }{{\cos \theta }}\frac{{x - {x_1}}}{{\sin \alpha }},{x_1} < x \le 2a\\
					\frac{{((x - 2a)C + \frac{\varepsilon }{{\cos \theta }})}}{2}A + \frac{{(xC + \frac{\varepsilon }{{\cos \theta }})}}{2}B \\
					+ \frac{\varepsilon }{{\cos \theta }}(\frac{\varepsilon }{{\cos \alpha }} - A - B),2a < x \le {x_1} + {x_2} + a
				\end{array} \right..	
				\label{equ 16}
			\end{equation}
			
			In the 3D case, RPE is the fraction of the volume of shaded space in cube $\#i$. As the value of $x$ increases, the shape of the shaded part will gradually change from a triangular block to a combination of a triangular block and a prismatic table. Here we take \eqref{equ 13} as an example, and first derive the volume of the prismatic table as follows.
			\begin{equation}
				V = \frac{1}{3}h\left( {{S_a} + {S_b} + \sqrt {{S_a} \times {S_b}} } \right),	
				\label{equ 17}
			\end{equation}
			where $S_a$ is the area of the top of the prismatic table, $S_b$ is the area of the bottom of the prismatic table, and $h=x\cos{\alpha}$ is the altitude of the prismatic table, i.e., the distance between vertex O and the boundary surface of licensed network. Substituting the value of $S_i\left(x\right)$ from \eqref{equ 13} into \eqref{equ 18}, \eqref{equ 19}, \eqref{equ 20} and \eqref{equ 21}, we have the following results.
			\begin{itemize}
				\item[$\bullet$]When $x\le2a$, the volume of the triangular block is
			\end{itemize}	
			\begin{equation}
				{V_1}(x) = \frac{1}{3}{S_i}\left( x \right)x\cos \alpha.	
				\label{equ 18}
			\end{equation}
			\begin{itemize}
				\item[$\bullet$]When $2a < x \le {x_1}$, the volume of the prismatic table is
			\end{itemize}
			\begin{equation}
				{V_2}(x) = \frac{1}{3}h\left( {{S_i}\left( {2a} \right) + {S_i}(x) + \sqrt {{S_i}(x) \times {S_i}(2a)} } \right),
				\label{equ 19}
			\end{equation}
			where $h = (x - 2a)\cos \alpha$.\\
			\begin{itemize}
				\item[$\bullet$]When ${x_1} < x \le {x_1} + 2a$, the volume of the prismatic table is
			\end{itemize}
			\begin{equation}
				{V_3}(x) = \frac{1}{3}h\left( {{S_i}\left( { {x_1}} \right) + {S_i}(x) + \sqrt {{S_i}(x) \times {S_i}({x_1})} } \right),
				\label{equ 20}
			\end{equation}
			where $h = (x - {x_1})\cos \alpha$.\\
			\begin{itemize}
				\item[$\bullet$]When ${x_1} + 2a < x \le {x_1} + {x_2} + a$, the volume of the prismatic table is
			\end{itemize}
			\begin{equation}
				{V_4}(x) = \frac{1}{3}h\left( {{S_i}\left( { {x_1} + 2a} \right) + Si(x) + \sqrt {Si(x) \times Si({x_1} + 2a)} } \right),	
				\label{equ 21}
			\end{equation}
			where $h = (x - {x_1} - 2a)\cos \alpha$.\\ Thus the RPE of cube $i$ in \eqref{equ 13} is
			\begin{equation}
				{P_{e,i}} = \left\{ \begin{array}{l}
					\frac{{{V_1}(x)}}{{{\varepsilon ^3}}},x \le 2a\\
					\frac{{{V_2}(x) + {V_1}(2a)}}{{{\varepsilon ^3}}},2a < x \le {x_1}\\
					\frac{{{V_3}(x) + {V_2}({x_1}) + {V_1}(2a)}}{{{\varepsilon ^3}}},{x_1} < x \le {x_1} + 2a\\
					\frac{{{V_4}(x) + {V_3}({x_1} + 2a) + {V_2}({x_1}) + {V_1}(2a)}}{{{\varepsilon ^3}}},\\
					{x_1} + 2a < x \le {x_1} + {x_2} + a
				\end{array} \right..	
				\label{equ 22}
			\end{equation}
			
			The expectation of $S_i$ and $P_{e,i}$ in \eqref{equ 13} are shown in \eqref{equ 36} and \eqref{equ 37}.
			
			\stripsep 1pt
			\begin{strip}
				\begin{equation}\label{equ 36}
					\begin{array}{l}
						{E_1}\left( {{S_i}} \right) = \int_0^{{{\tan }^{ - 1}}\frac{{\sqrt 2 }}{2}} {\left( {\int_{{{\tan }^{ - 1}}\frac{{\sqrt 2 }}{2}}^{\frac{\pi }{4} - {{\sin }^{ - 1}}\tan \alpha } {\left( \begin{array}{l}
										\int_0^{\varepsilon \tan \alpha } {\frac{{{x^2}\left( {\tan \theta  + \cot \theta } \right)}}{{2\sin \alpha }}{f_X}\left( x \right)} dx + \int_{\varepsilon \tan \alpha }^{\varepsilon \sin \theta } {\frac{{\left( {x - a} \right)\left( {\tan \theta  + \cot \theta } \right)\varepsilon }}{{\cos \alpha }}{f_X}\left( x \right)} dx\\
										+ \int_{\varepsilon \sin \theta }^{\varepsilon \sin \theta  + \varepsilon \tan \alpha } {\left( \begin{array}{l}
												\frac{{\left( {x - 2a} \right)\left( {\tan \theta  + \cot \theta } \right) + \frac{\varepsilon }{{\cos \theta }}}}{2}A\\
												+ \frac{\varepsilon }{{\cos \theta }}\left( {\frac{\varepsilon }{{\cos \alpha }} - A} \right)
											\end{array} \right)} {f_X}\left( x \right)dx\\
										+ \int_{\varepsilon \sin \theta  + \varepsilon \tan \alpha }^{\frac{{\varepsilon \left( {\cos \theta  + \sin \theta  + \tan \alpha } \right)}}{2}} {\frac{\varepsilon }{{\cos \theta }}\frac{\varepsilon }{{\cos \alpha }}} {f_X}\left( x \right)dx
									\end{array} \right){f_\Theta }\left( \theta  \right)d\theta } } \right)} {f_{\rm A}}(\alpha )d\alpha \\
						%\mgin	
						= \int_0^{ta{n^{ - 1}}\frac{{\sqrt 2 }}{2}} {\left( {\int_{{{\sin }^{ - 1}}\tan \alpha }^{\frac{\pi }{4} - {{\sin }^{ - 1}}\tan \alpha } {\left( \begin{array}{l}
										{\varepsilon ^2}{F_X}\left( {\varepsilon \tan \alpha } \right)\frac{{\left( {\tan \theta  + \cot \theta } \right)\sin \alpha }}{{2{{\cos }^2}\alpha }} + {\varepsilon ^2}{F_X}\left( {\varepsilon \sin \theta } \right)\frac{{2\sin \theta  - \tan \alpha }}{2}\\
										- {\varepsilon ^2}{F_X}\left( {\varepsilon \tan \alpha } \right)\frac{{\tan \alpha }}{2} + {\varepsilon ^2}\left( \begin{array}{l}
											{F_X}\left( {\frac{{\varepsilon \left( {\cos \theta  + \sin \theta  + \tan \alpha } \right)}}{2}} \right)\\
											- {F_X}\left( {\varepsilon \sin \theta } \right)
										\end{array} \right)\frac{1}{{\cos \theta \cos \alpha }}\\
										- {\varepsilon ^2}\left( \begin{array}{l}
											{F_X}\left( {\varepsilon \sin \theta  + \varepsilon \tan \alpha } \right)\\
											- {F_X}\left( {\varepsilon \sin \theta } \right)
										\end{array} \right)\frac{{\left( {\sin \theta  + \tan \alpha } \right)\left( {1 + \cos \theta \tan \alpha \left( {\tan \theta  + \cot \theta } \right)} \right)}}{{2\cos \theta \sin \alpha }}\\
										+ {\varepsilon ^2}\left( \begin{array}{l}
											\left( {\sin \theta  + \tan \alpha } \right)\\
											{F_X}\left( {\varepsilon \sin \theta  + \varepsilon \tan \alpha } \right)\\
											- \sin \theta {F_X}\left( {\varepsilon \sin \theta } \right)
										\end{array} \right)\left( \begin{array}{l}
											\frac{{\left( {\tan \theta  + \cot \theta } \right)}}{{2\sin \alpha }}\left( {\sin \theta  + 2\tan \alpha } \right)\\
											+ \frac{1}{{2\cos \theta \sin \alpha }}
										\end{array} \right)\\
										+ {\varepsilon ^2}\frac{{\left( {\tan \theta  + \cot \theta } \right)}}{{2\sin \alpha }}\left( \begin{array}{l}
											{\sin ^2}\theta {F_X}\left( {\varepsilon \sin \theta } \right)\\
											- {\left( {\sin \theta  + \tan \alpha } \right)^2}{F_X}\left( {\varepsilon \sin \theta  + \varepsilon \tan \alpha } \right)
										\end{array} \right)\\
										+ \varepsilon \left( {\frac{{\left( {\tan \theta  + \cot \theta } \right)}}{{2\sin \alpha }}\left( {\sin \theta  + 2\tan \alpha } \right) + \frac{1}{{2\cos \theta \sin \alpha }}} \right)\int_{\varepsilon \sin \theta }^{\varepsilon \sin \theta  + \varepsilon \tan \alpha } {{F_X}\left( x \right)dx} \\
										- \varepsilon \int_{\varepsilon \tan \alpha }^{\varepsilon \sin \theta } {{F_X}\left( x \right)dx}  + \frac{{\left( {\tan \theta  + \cot \theta } \right)}}{{2\sin \alpha }}\int_{\varepsilon \sin \theta }^{\varepsilon \sin \theta  + \varepsilon \tan \alpha } {2x{F_X}\left( x \right)dx} \\
										- \frac{{\left( {\tan \theta  + \cot \theta } \right)}}{{2\sin \alpha }}\int_0^{\varepsilon \tan \alpha } {2x{F_X}\left( x \right)dx}
									\end{array} \right)} {f_\Theta }\left( \theta  \right)d\theta } \right)} {f_{\rm A}}(\alpha )d\alpha\\
						
						= {\varepsilon ^2}{Q_1}{F_{X1}}\left( \varepsilon \right).
					\end{array}
				\end{equation}
			\end{strip}
			\begin{strip}
				\begin{equation}\label{equ 37}
					\begin{array}{l}
						{E_1}\left( {{P_{e,i}}} \right) = \int_0^{{{\tan }^{ - 1}}\frac{{\sqrt 2 }}{2}} {\left( {\int_{{{\sin }^{ - 1}}\tan \alpha }^{\frac{\pi }{4} - {{\sin }^{ - 1}}\tan \alpha } {\left( \begin{array}{l}
										\int_0^{\varepsilon \tan \alpha } {\frac{{{V_1}\left( x \right)}}{{{\varepsilon ^3}}}{f_X}\left( x \right)dx}  + \int_{\varepsilon \tan \alpha }^{\varepsilon \sin \theta } {\frac{{{V_2}\left( x \right) + {V_1}\left( {2a} \right)}}{{{\varepsilon ^3}}}{f_X}\left( x \right)} dx\\
										+ \int_{\varepsilon \sin \theta }^{\varepsilon \sin \theta  + \varepsilon \tan \alpha } {\frac{{{V_3}\left( x \right) + {V_2}\left( {{x_1}} \right) + {V_1}\left( {2a} \right)}}{{{\varepsilon ^3}}}} {f_X}\left( x \right)dx\\
										+ \int_{\varepsilon \sin \theta  + \varepsilon \tan \alpha }^{\frac{{\varepsilon \left( {\cos \theta  + \sin \theta  + \tan \alpha } \right)}}{2}} {\frac{{{V_4}\left( x \right) + {V_3}\left( {{x_1} + 2a} \right) + {V_2}\left( {{x_1}} \right) + {V_1}\left( {2a} \right)}}{{{\varepsilon ^3}}}} {f_X}\left( x \right)dx
									\end{array} \right)} {f_\Theta }\left( \theta  \right)d\theta } \right)} {f_{\rm A}}(\alpha )d\alpha\\
						= {Q_2}{F_{X2}}\left( \varepsilon \right).
					\end{array}
				\end{equation}
			\end{strip}
			\begin{figure}[t]
				\centering
				\includegraphics[width=0.26\textheight]{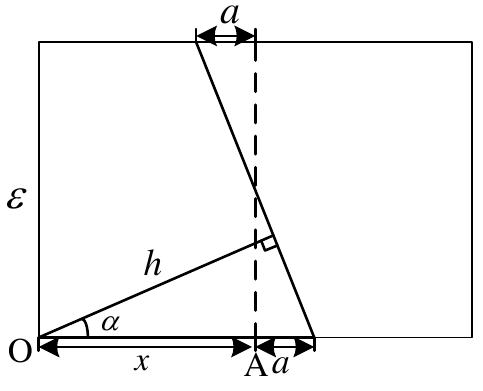}
				\DeclareGraphicsExtensions.
				\caption{The left view of cube $\#i$.}
				\label{fig 4}
			\end{figure}
			Because the upper and lower limits of the integral with respect to $\theta$ and $\alpha$ are constants in \eqref{equ 36} and \eqref{equ 37}, only the integral about $x$ is calculated. Meanwhile, since the side length of cube $\varepsilon$ is a constant,$\ F_{X1}\left(\varepsilon\right)$ and $F_{X2}\left(\varepsilon\right)$ are constants related to $\varepsilon$, and $Q_1$ and $Q_2$ are constant.
			
			Similarly, the expectation of $S_i$ in \eqref{equ 14}, \eqref{equ 15} and \eqref{equ 16} are represented as $E_2\left(S_i\right)$,  $E_3\left(S_i\right)$ and $E_4\left(S_i\right)$, and the expectation of $P_{e,i}$ in \eqref{equ 14}, \eqref{equ 15} and \eqref{equ 16} are represented as $E_2\left(P_{e,i}\right)$, $E_3\left(P_{e,i}\right)$ and $E_4\left(P_{e,i}\right)$. Thus the comprehensive expectation of $S_i$ and $P_{e,i}$ in cube $\#i$ are
			\begin{figure}[t]
				\centering
				\includegraphics[width=0.36\textheight]{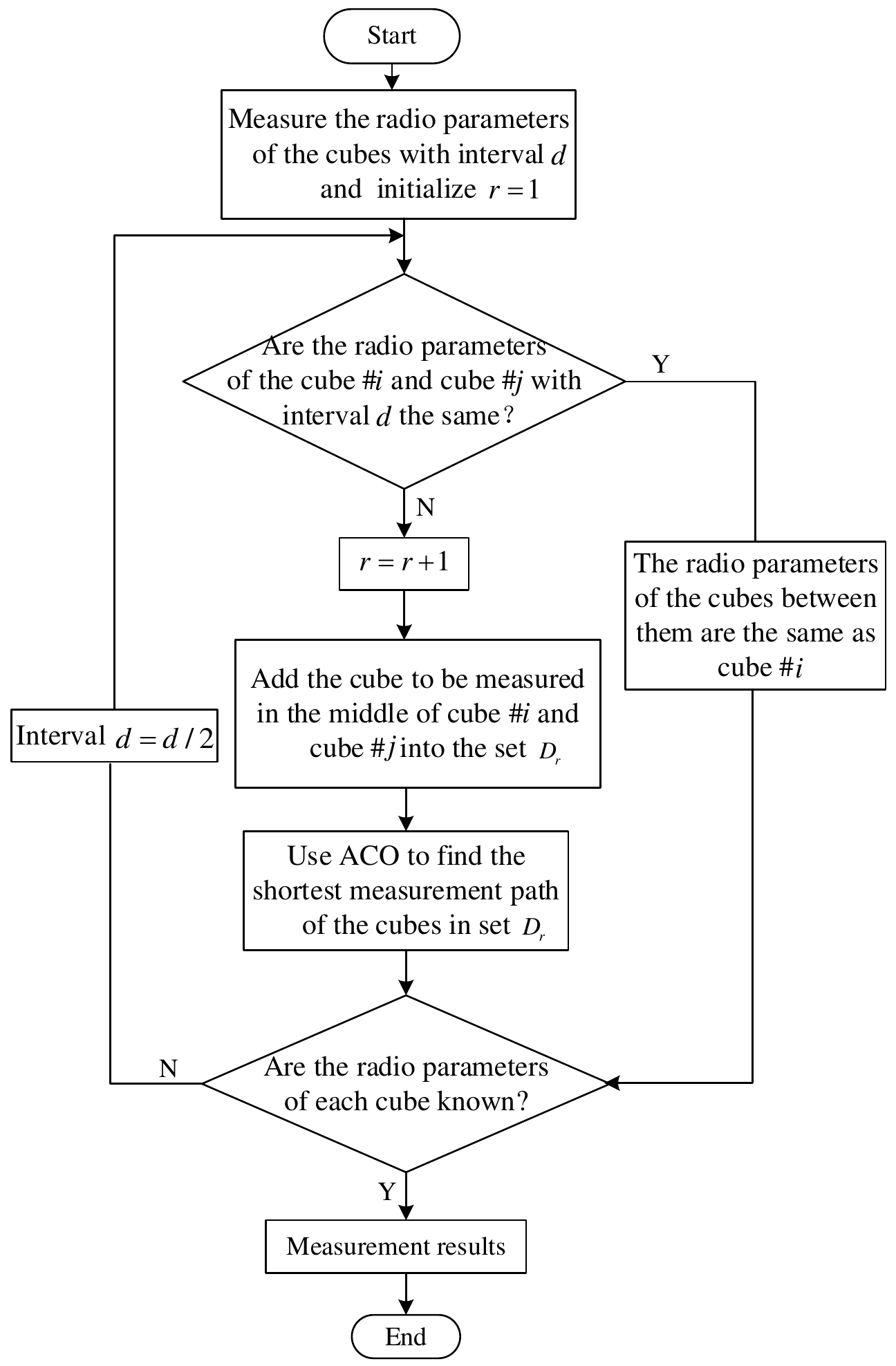}
				\DeclareGraphicsExtensions.
				\caption{3D SOM algorithm.}
				\label{fig 5}
			\end{figure}
			\begin{figure}[t]
				\centering
				\includegraphics[width=0.22\textheight]{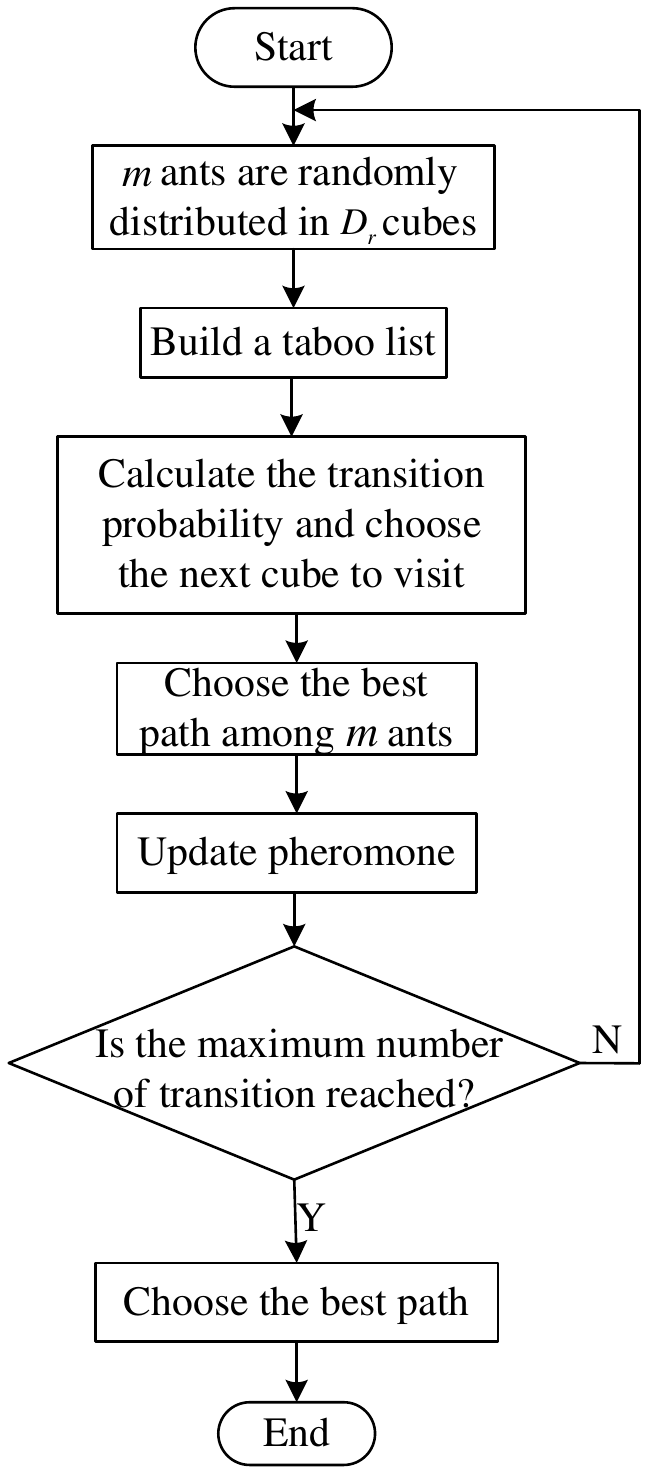}
				\DeclareGraphicsExtensions.
				\caption{Path planning algorithm of UAV based on ACO.}
				\label{fig 6}
			\end{figure}
			\begin{equation}
				E\left( {{S_i}} \right) = {E_1}\left( {{S_i}} \right) + {E_2}\left( {{S_i}} \right) + {E_3}\left( {{S_i}} \right) + {E_4}\left( {{S_i}} \right),		
				\label{equ 25}
			\end{equation}
			\begin{equation}
				E\left( {{P_{e,i}}} \right) = {E_1}\left( {{P_{e,i}}} \right) + {E_2}\left( {{P_{e,i}}} \right) + {E_3}\left( {{P_{e,i}}} \right) + {E_4}\left( {{P_{e,i}}} \right).	
				\label{equ 26}
			\end{equation}
			
			If the radio environment of the first $K$ cubes is impure among the $M$ cubes, then
			\begin{equation}
				E\left( {{S_i}} \right)\mathop  = \limits^{(a)} \frac{1}{K}\sum\limits_{i = 1}^K {{S_i}} \mathop  = \limits^{(b)} \frac{1}{K}S,	
				\label{equ 27}
			\end{equation}
			where $(a)$ is due to the Law of Large Numbers (LLN), and $(b)$ is due to the fact that the cubes with impure radio environment cover all the boundary surfaces of licensed networks. The value of $K$ is estimated as
			\begin{equation}
				K = \frac{S}{{E\left( {{S_i}} \right)}}.	
				\label{equ 28}
			\end{equation}
			
			The RPE of the entire 3D region is
			\begin{equation}
				{P_e} = \frac{1}{M}\sum\limits_{i = 1}^K {{P_{e,i}}}  = \frac{K}{M}E\left( {{P_{e,i}}} \right).	
				\label{equ 29}
			\end{equation}
			Substituting the value of $K$ from (28) and the value of $E\left(P_{e,i}\right)$ from \eqref{equ 26} into \eqref{equ 29}, we have
			\begin{equation}
				{P_e} = S\frac{1}{M}\frac{{E\left( {{P_{e,i}}} \right)}}{{E\left( {{S_i}} \right)}}\; = \;Q\;S\frac{1}{{{\varepsilon ^2}M}} = \;Q\frac{S}{{{L^2}}}\frac{1}{{\sqrt[3]{M}}} = \Theta \left( {\frac{1}{{\sqrt[3]{M}}}} \right),	
				\label{equ 30}
			\end{equation}
			where $Q$ is a constant, $S$ is the area of all the boundary surfaces of licensed networks, $L$ is the length of the 3D space edge, $K = Q\frac{S}{{{L^2}}}$ is the number of the cubes with impure radio environment, and $M$ is the number of cubes.
		\end{proof}
		
		Thus, we obtain the relation between the error of 3D SOM and the number of measurements when the shape of licensed network's coverage is irregular. An increase of (a reduction of) the number of measurements will reduce (increase) the error of 3D SOM. A specific case is shown in the following theorem.
		\begin{Theo}
			If the parameters related with the shape of licensed network's coverage, namely, the distance between vertex $O$ and the intersection line between the boundary surface of licensed network and the cube's bottom ($x$), the angle between the intersection line between the boundary surface of licensed network and horizontal line ($\theta$), and the angle between the boundary surface of licensed network and the bottom of the cube ($\alpha$), follow uniform distribution, the RPE is expressed as
			\begin{equation}
				{P_e} = K\frac{1}{{\sqrt[3]{M}}},	
				\label{equ 31}
			\end{equation}
			where $K = 0.1649\frac{S}{{{L^2}}}$ ,which is a constant determined by the area of all the boundary surfaces of licensed networks $S$ and the length of the 3D space edge $L$.
		\end{Theo}
		\begin{proof}
			If the parameters related with the shape of licensed network's coverage are uniformly distributed, then $x$, $\theta$  and $\alpha$ are independent variables with probability density functions (PDFs)
			\begin{equation}
				{f_X}\left( x \right) = \frac{1}{{\frac{{\sqrt 2 }}{2}\varepsilon \sin \left( {\theta  + \frac{\pi }{4}} \right)}},0 \le x \le \frac{{\sqrt 2 }}{2}\varepsilon \sin \left( {\theta  + \frac{\pi }{4}} \right),	
				\label{equ 32}
			\end{equation}
			\begin{equation}
				{f_\Theta }\left( \theta  \right) = \frac{4}{\pi },0 \le x \le \frac{\pi }{4},		
				\label{equ 33}
			\end{equation}
			\begin{equation}
				{f_{\rm A}}(\alpha ) = \frac{4}{\pi },0 \le x \le \frac{\pi }{4}.	
				\label{equ 34}
			\end{equation}
			Substituting ${f_X}\left( x \right)$ in \eqref{equ 32}, ${f_\Theta }\left( \theta  \right)$ in \eqref{equ 33} and ${f_{\rm A}}(\alpha )$ in \eqref{equ 34} into \eqref{equ 7}, \eqref{equ 8} and \eqref{equ 9} respectively, we have
			\begin{equation}
				{P_e} = 0.1649\frac{S}{{{L^2}}}\frac{1}{{\sqrt[3]{M}}}.	
				\label{equ 35}
			\end{equation}
		\end{proof}
		
		As revealed in \textbf{Theorem 2}, the precise result is obtained when the parameters related with the shape of licensed network's coverage are uniformly distributed. Thus, the tradeoff between accuracy and efficiency in 3D SOM is revealed.

		\section{3D SOM Algorithm using UAV}
		
		In Section II and Section III, we study the representation of radio parameters and the performance of 3D SOM. In this section, we design a 3D SOM algorithm using UAV. Besides, the path planning algorithm of UAV is designed for the SOM using UAV.
		
		\subsection{3D SOM Algorithm}
		The UAV is applied for 3D SOM. Firstly, the entire 3D space is divided into small cubes. Then, the UAV is applied to measure the radio parameters in the cubes. Since the UAV measures a cube once and regards the measurement result as the radio parameter of the cube, there exists a probability to have a wrong measurement. However, the increase of $M$ could reduce the measurement error. As the correlation between the adjacent cubes in the space is very large, if each cube in the space is measured, there will be a large amount of redundancy, which will cause the waste of measurement time. Facing this problem, we designed the 3D SOM algorithm and path planning algorithm of UAV to reduce the measurement time, which are shown in Fig. \ref{fig 5} and Fig. \ref{fig 6}, respectively.

		In the designed 3D SOM algorithm, the cubes with interval $d$ are firstly added into the set $D_r$ and $r$ is initialized as 1, where $r$ is the number of iterations. The radio parameters of the cubes in the set $D_r$ are measured and recorded. Then, the following steps are implemented.
		\begin{itemize}
			\item[$\bullet$]\textbf{\emph{Step 1}}: If the radio parameters of cube \#$i$ and cube \#$j$ with interval $d$ are the same, the radio parameters of the cubes between them are the same as cube \#$i$. Otherwise, set $r=r+1$,  and the cube to be measured in the middle of cube \#$i$ and cube \#$j$ are added into the set $D_r$, and the cubes in the set $D_r$ will be measured in Step 2.
			\item[$\bullet$]\textbf{\emph{Step 2}}: The path planning algorithm of UAV to visit the cubes in the set $D_r$ is designed. The Ant Colony Optimization (ACO) \cite{ref24},\cite{ref25} is applied in solving the path planning problem, and therefore we utilize the ACO to find the shortest measurement path of the cubes in set $D_r$, as shown in Fig. \ref{fig 6}. Then, set $d=d/2$.
			\item[$\bullet$]If the radio parameters of at least one cube are unknown, repeat Step 1 and Step 2. Otherwise, the 3D SOM is completed, namely, the radio parameter of every cube is known in the entire 3D space.
		\end{itemize}
		
		The designed SOM algorithm can reduce the redundancy of measurement, and the path planning algorithm is applied to reduce the flight distance of UAV.
		\begin{figure}[t]
			\centering
			\includegraphics[width=0.37\textheight]{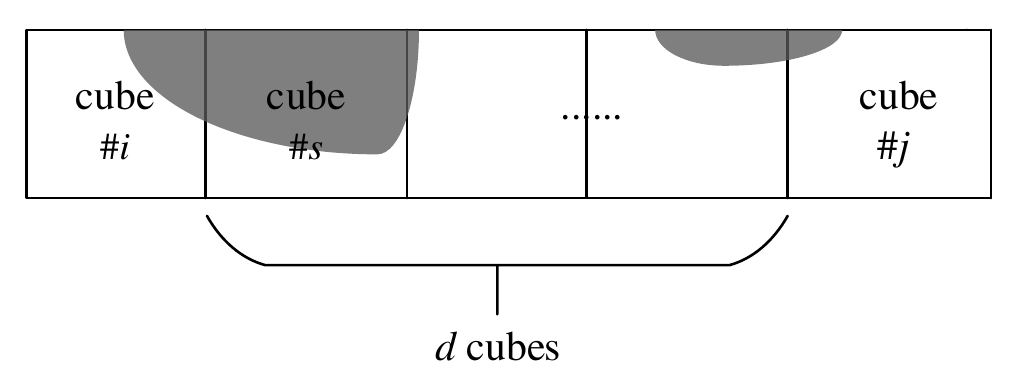}
			\DeclareGraphicsExtensions.
			\caption{The measurement error.}
			\label{fig 7}
		\end{figure}
		\begin{figure}[t]
			\centering
			\includegraphics[width=0.33\textheight]{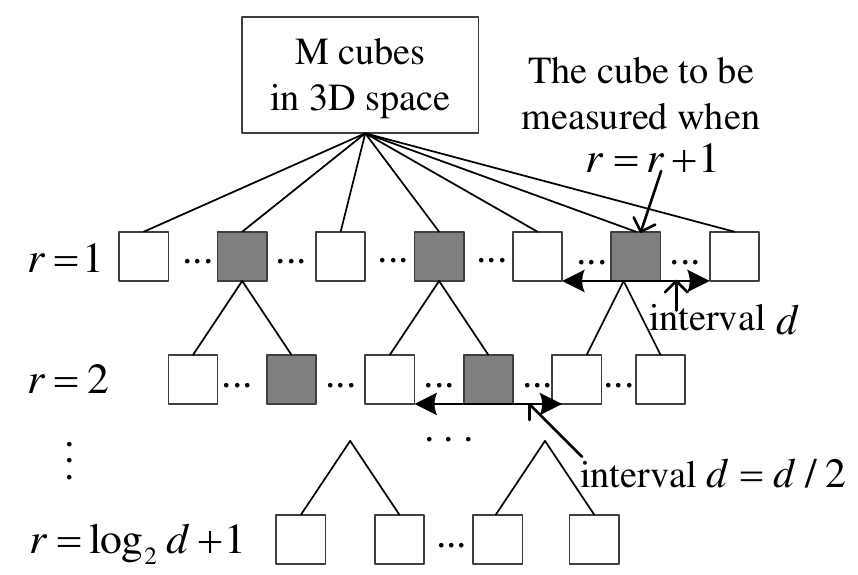}
			\DeclareGraphicsExtensions.
			\caption{Measurement process.}
			\label{fig 8}
		\end{figure}
		
		\subsection{Analysis of 3D SOM algorithm}	
		
		According to the designed 3D SOM algorithm, if the radio parameters of the cubes with interval $d$ are the same, the radio parameters of the cubes between them are also the same. As illustrated in Fig. \ref{fig 7}, the radio parameters of the cubes \#$i$ and \#$j$ are the same according to the definition of radio parameters in (4), but the radio parameter of cube \#$s$ is different from cube \#$i$. Obviously, there will be measurement error, which becomes larger with the increase of the interval $d$.
		
		As to the performance of the proposed 3D SOM algorithm, we have the following theorem.
		
		\begin{Theo}
			In the 3D space, the upper bound of the number of measurements using the designed 3D SOM algorithm is
			\begin{equation}
				{D_d} \le \frac{M}{{{d^3}}} + \frac{{2\sqrt 3 S}}{{3{{\left( {\varepsilon L} \right)}^2}}}\left( {1 - \frac{1}{{{d^2}}}} \right){M^{\frac{2}{3}}},	
				\tag{36}
			\end{equation}
			where $M$ is the number of cubes, $S$ is the area of all the boundary surfaces of licensed networks and $L$ is the length of the 3D space's edge. If the interval $d > 1, d \in N$, the number of measurements using the designed 3D SOM algorithm is smaller than that using regular measurement algorithm, which is the strategy that UAV follows a zig-zag pattern across every cube of a 3D space denoted as the Snake Traversal \cite{ref26}.
		\end{Theo}
		\begin{proof}
			The cubes with impure radio environment are distributed along the boundary surface of licensed networks. The side length of a cube is denoted by $\varepsilon$, and we have $M=\left(\frac{L}{\varepsilon}\right)^3$.
			
			The measurement process is illustrated using tree graph in Fig. \ref{fig 8}, the entire measurement process consists of $\log_2{d}+1$ iterations. With $r$ denoting the number of iteration, when $r=1$, the cubes with interval $d$ in 3D space are measured, and there are $D_1=\frac{M}{d^3}$ cubes that need to be measured. When $r>1$, the cubes need to be measured are denoted by the gray cubes as shown in Fig. \ref{fig 8}, and the number of measurements is equal to the number of impure cubes along the boundary surfaces of licensed networks, which can be derived by analyzing a packing problem. Moving the boundary surfaces of licensed networks in the opposite two normal directions with a distance $\frac{\sqrt3\varepsilon d}{2^{\left(r-2\right)}}$ generates a space with volume $\frac{2\sqrt3\varepsilon dS}{2^{\left(r-2\right)}}$ when $r>1$, where $S$ is the area of the boundary surfaces of licensed networks. This volume is divided by the volume  of the cube with volume $\left(\frac{\varepsilon d}{2^{\left(r-2\right)}}\right)^3$ to derive the maximum number of cubes with impure radio parameters $D_r=\frac{2\sqrt3\varepsilon dS/2^{\left(r-2\right)}}{\left(\varepsilon d/2^{\left(r-2\right)}\right)^3}=\frac{2^{\left(2r-3\right)}\sqrt3S}{\left(\varepsilon d\right)^2}$, i.e., the maximum number of measurements. Since the maximum of $r$ is $\log_2{d}+1$, the number of cubes need to be measured using the designed 3D SOM algorithm is
			\begin{equation}
				\begin{aligned}
					{D_{max}}& = {D_1} + {D_2} + ... + {D_r}\\
					& = \frac{M}{{{d^3}}} + \frac{{2\sqrt 3 S}}{{3{{\left( {\varepsilon L} \right)}^2}}}\left( {1 - \frac{1}{{{d^2}}}} \right){M^{\frac{2}{3}}}.
				\end{aligned}
				\tag{37}
			\end{equation}
		\end{proof}
	\begin{figure}[!htbp]
		\centering
		\subfigure[{Cubes to be measured with unit of each axis: m.}]{
			\label{fig 9a}
			\includegraphics[width=0.22\textwidth]{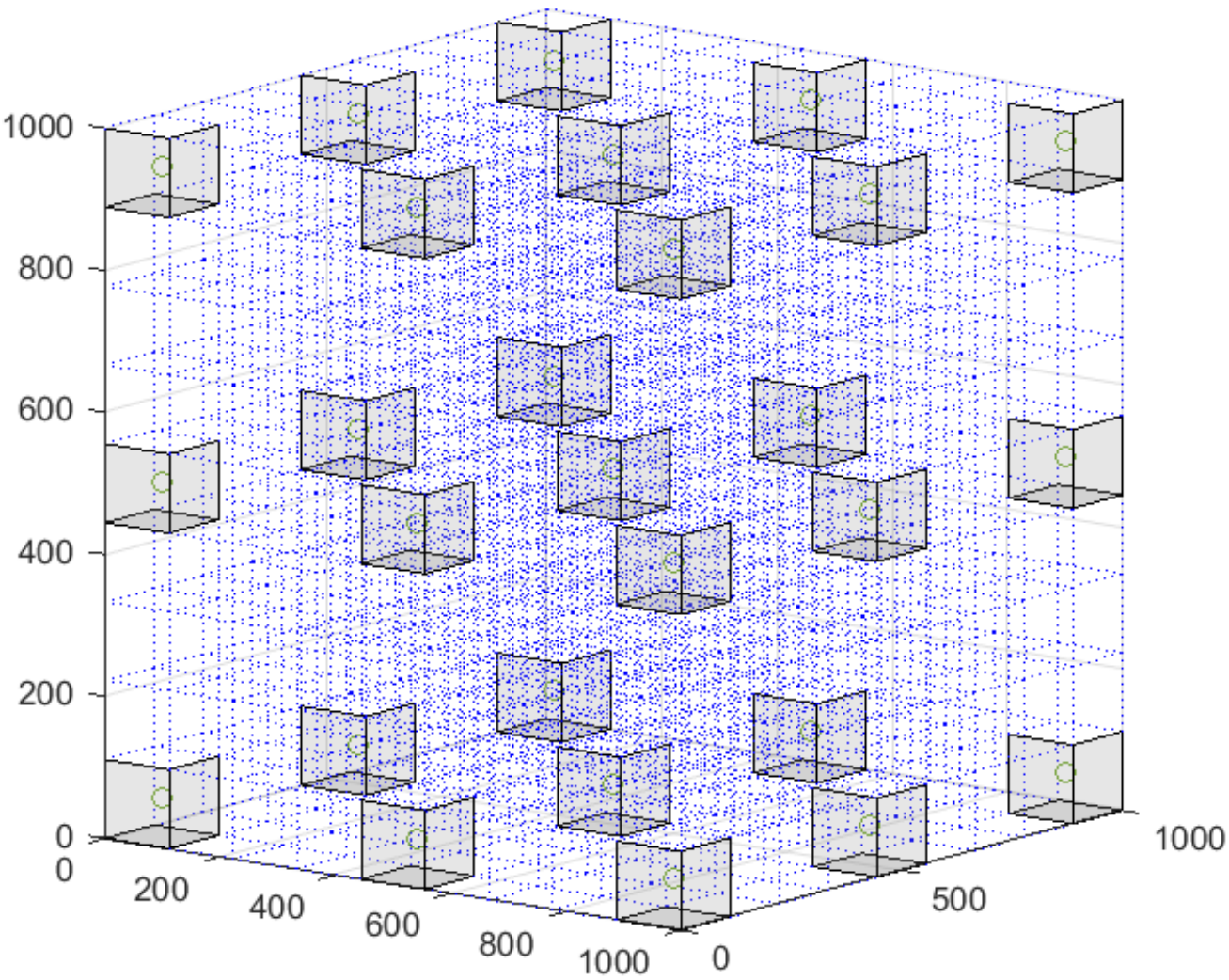}}
		\subfigure[{The flight path of UAV for SOM with unit of each axis: m.}]{
			\label{fig 9b}
			\includegraphics[width=0.22\textwidth]{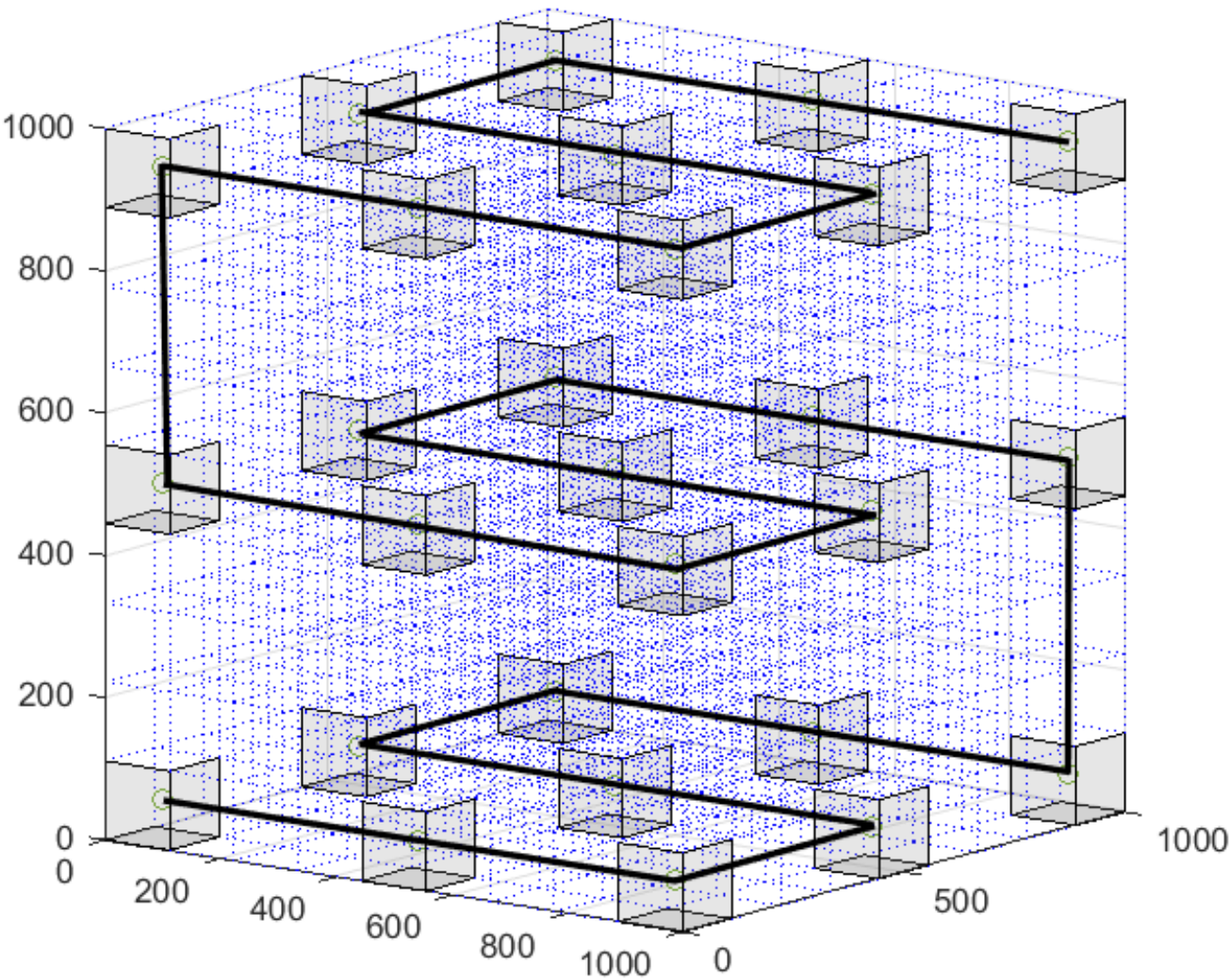}}
		\caption{The measurement process when $r=1$. }
		\label{fig 9}
	\end{figure}
	\begin{figure}[!htbp]
		\centering
		\subfigure[{Cubes to be measured with unit of each axis: m.}]{
			\label{fig 10a}
			\includegraphics[width=0.22\textwidth]{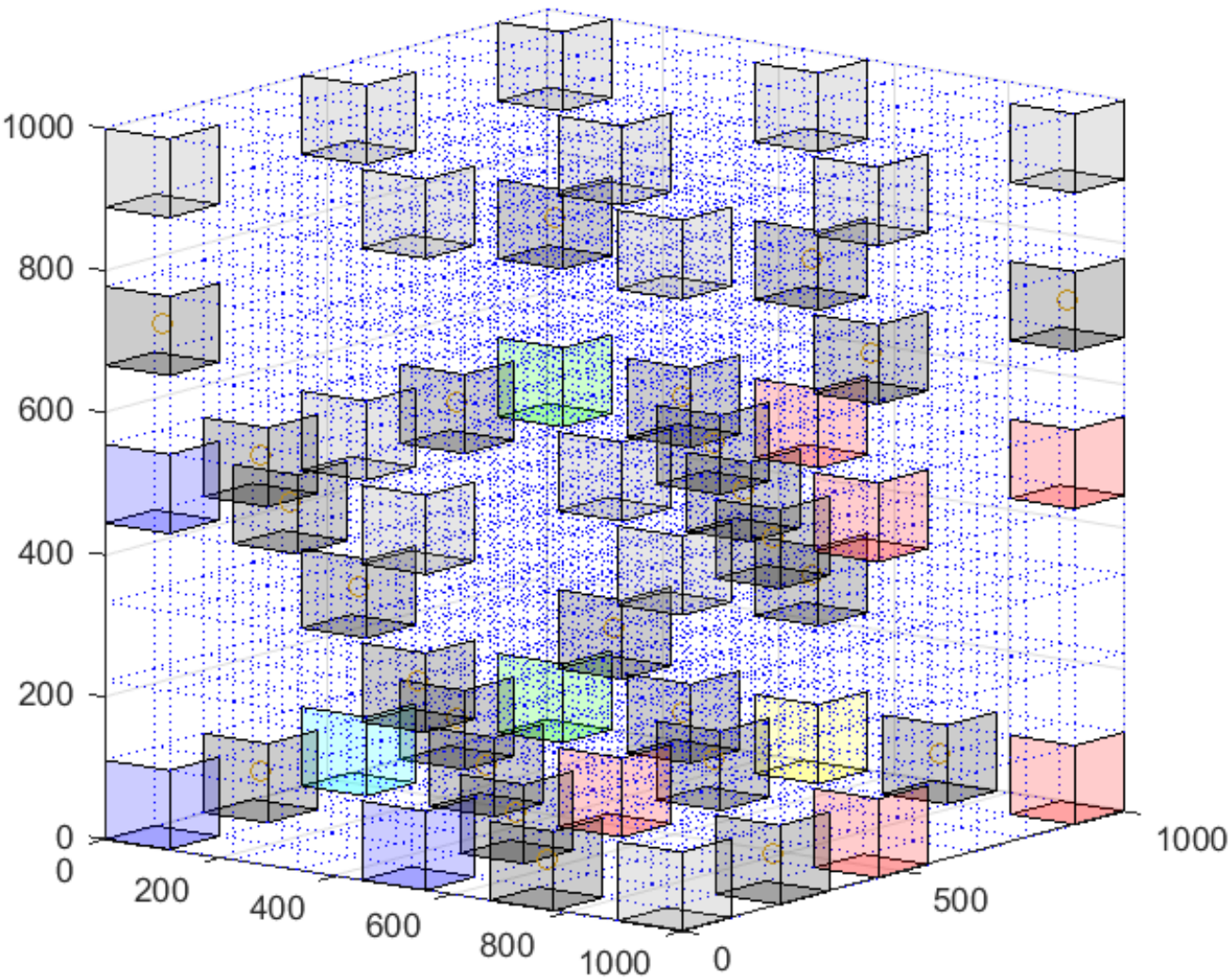}}
		\subfigure[{The flight path of UAV for SOM with unit of each axis: m.}]{
			\label{fig 10b}
			\includegraphics[width=0.22\textwidth]{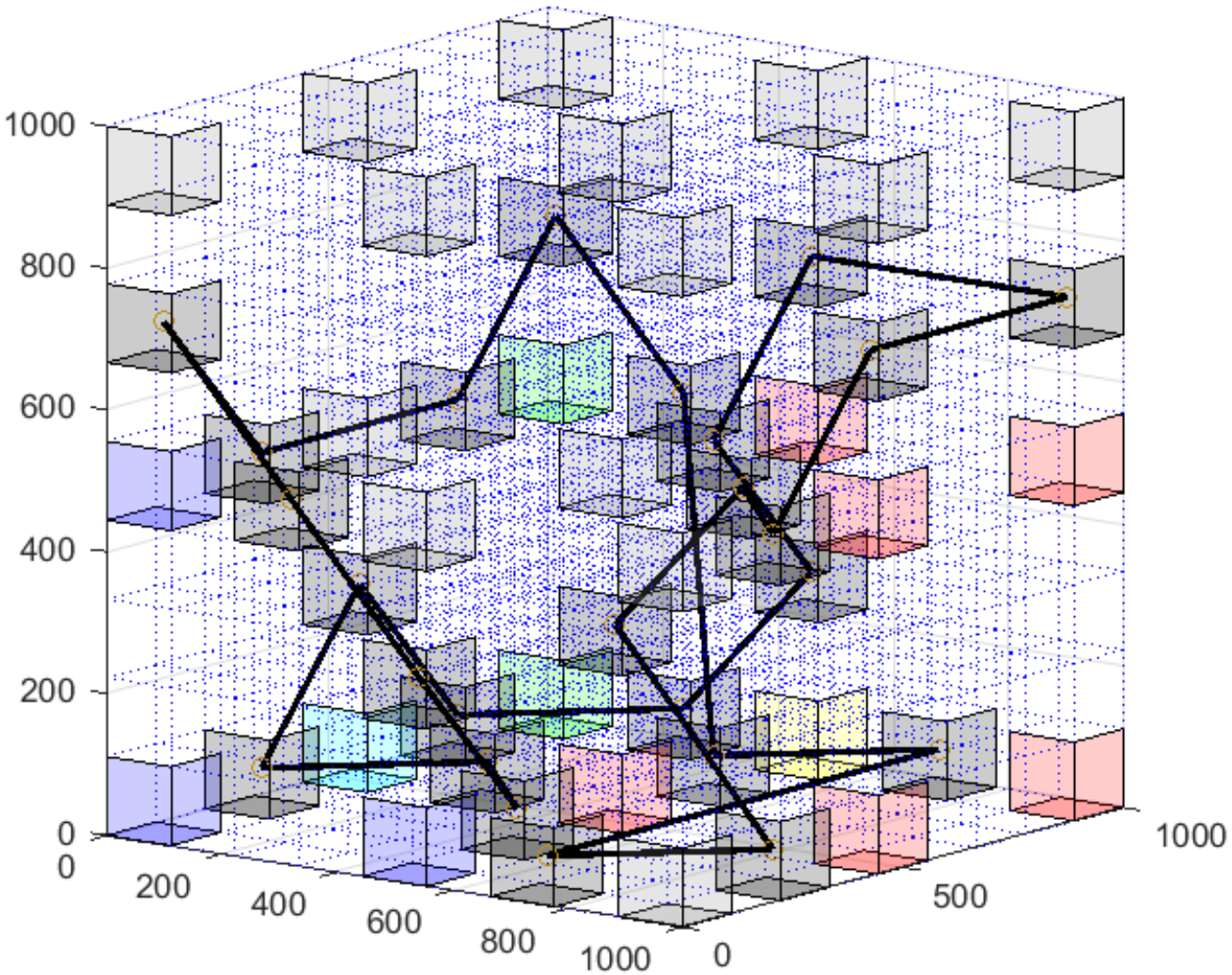}}
		\caption{The measurement process when $r=2$. }
		\label{fig 10}
	\end{figure}
	\begin{figure}[!htbp]
		\centering
		\subfigure[{Cubes to be measured with unit of each axis: m.}]{
			\label{fig 11a}
			\includegraphics[width=0.22\textwidth]{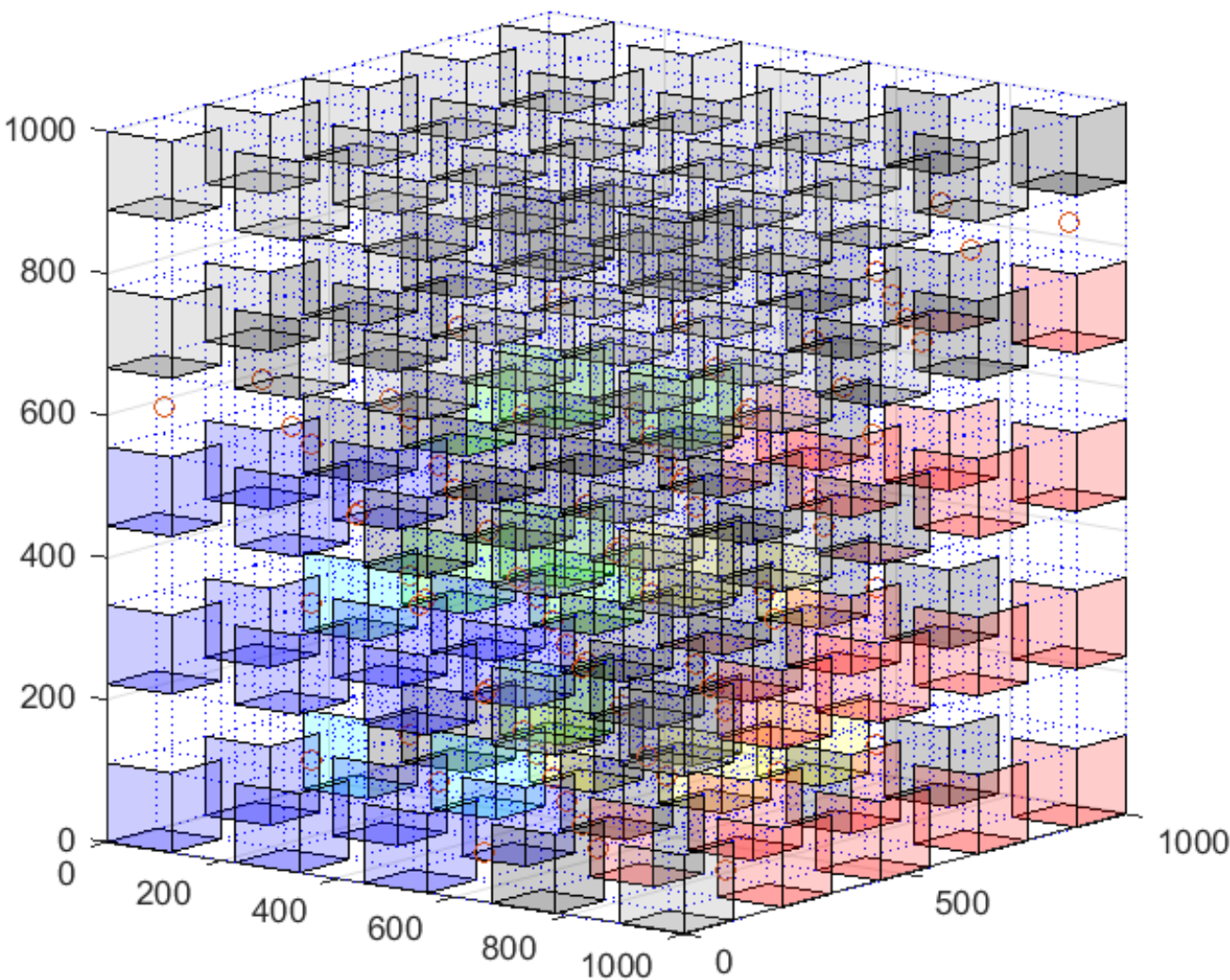}}
		\subfigure[{The flight path of UAV for SOM with unit of each axis: m.}]{
			\label{fig 11b}
			\includegraphics[width=0.22\textwidth]{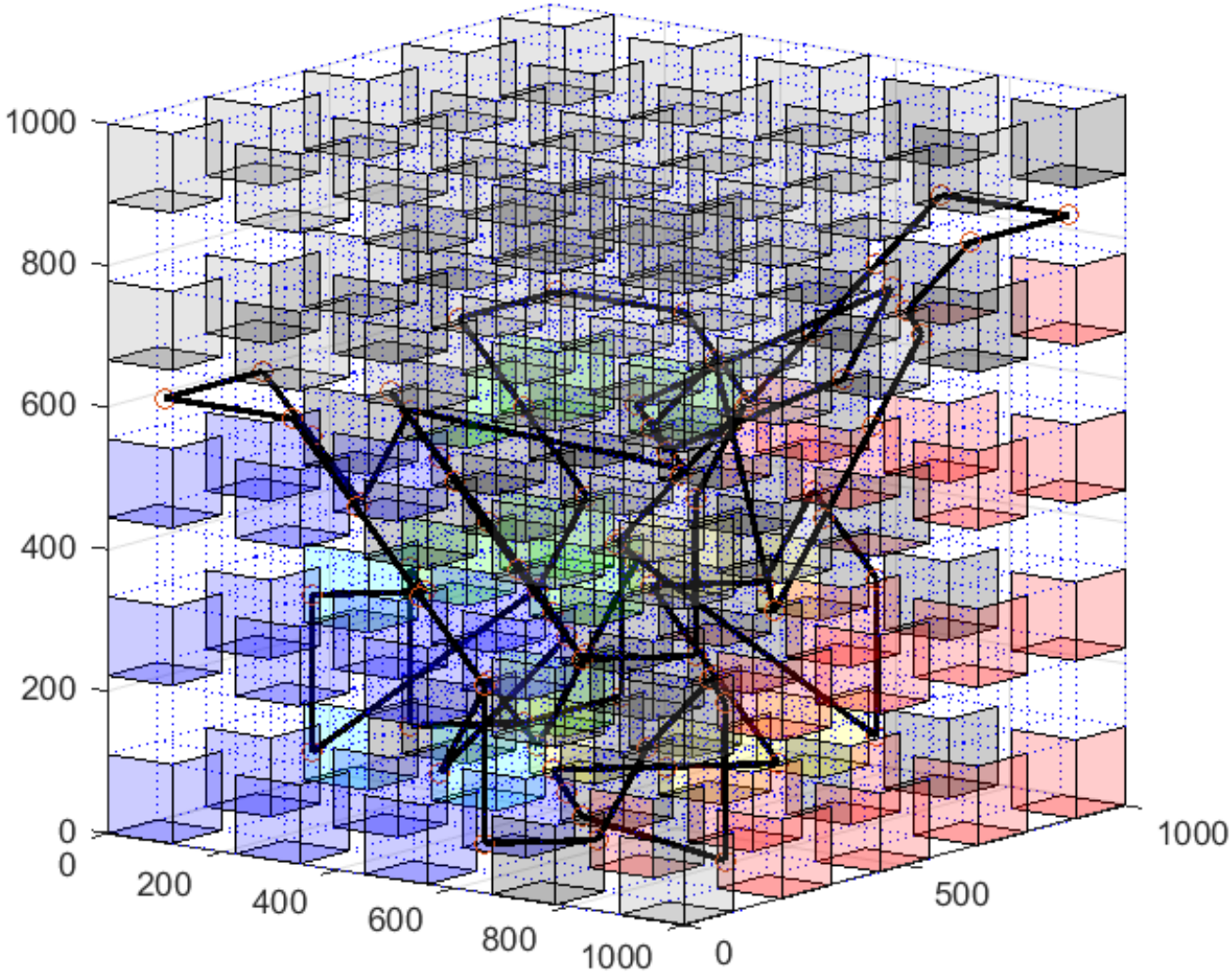}}
		\caption{The measurement process when $r=3$. }
		\label{fig 11}
	\end{figure}
		According to \textbf{Theorem 3}, the proposed 3D SOM algorithm can reduce the number of measurements compared with regular measurement algorithm when the interval $d>1$. Besides, the number of measurements decreases with the increase of $d$.
		
		\section{ Numerical Results}
		
		\begin{figure}[t]
			\centering
			\includegraphics[width=0.33\textheight]{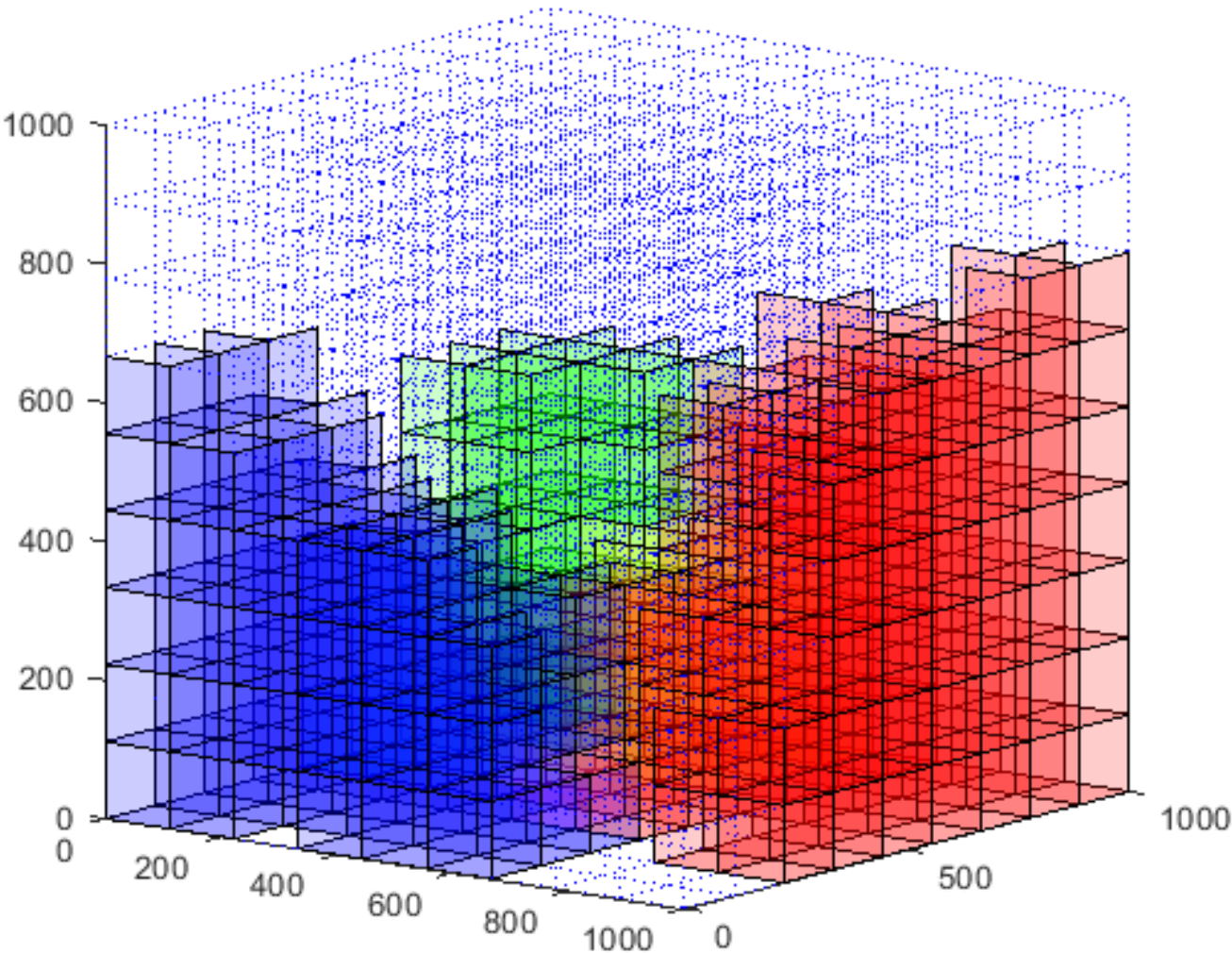}
			\DeclareGraphicsExtensions.
			\caption{3D SOM result with $M=9^3$ cubes with unit of each axis: m.}
			\label{fig 12}
		\end{figure}
		\begin{figure}[t]
			\centering
			\includegraphics[width=0.33\textheight]{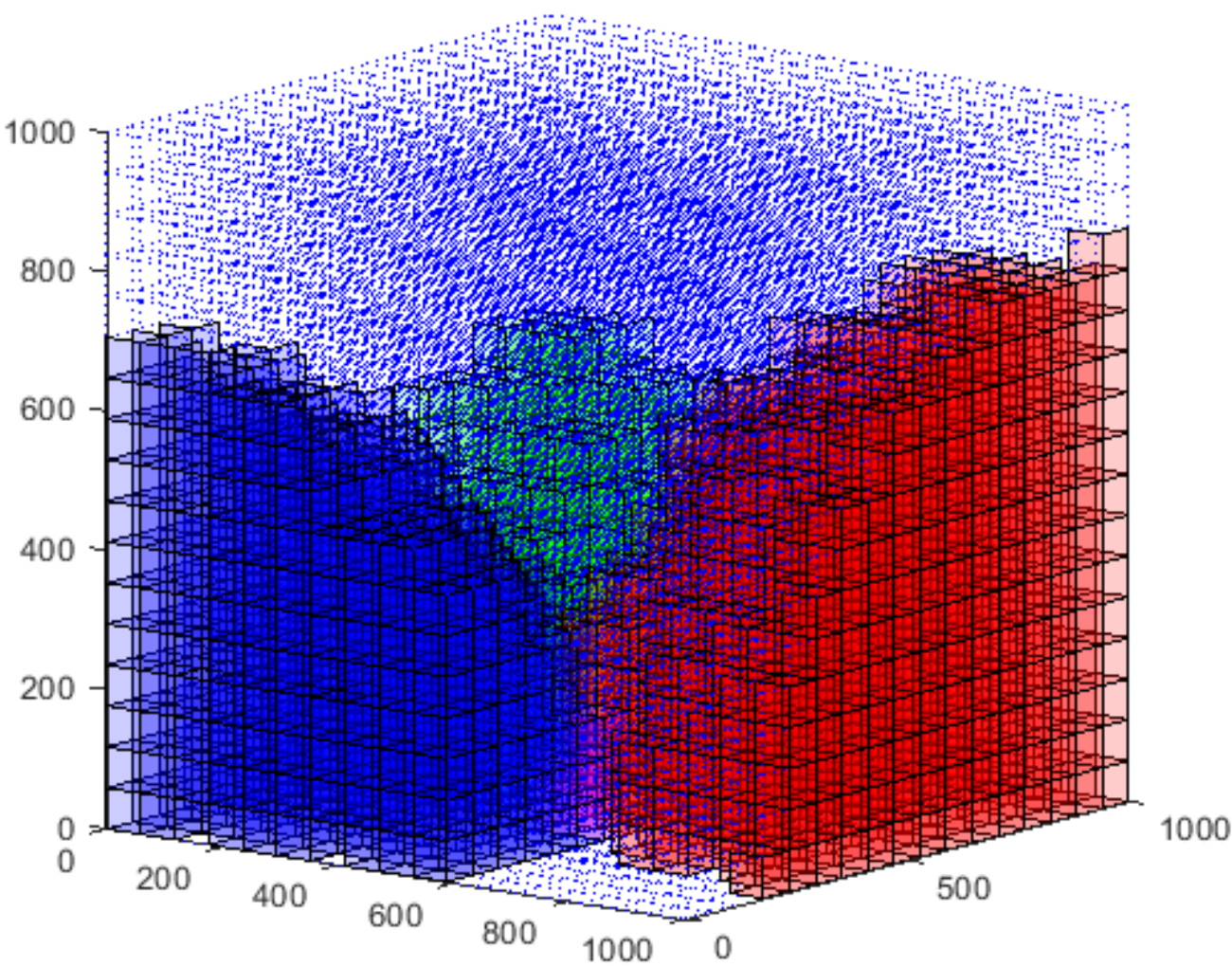}
			\DeclareGraphicsExtensions.
			\caption{3D SOM result with $M=17^3$ cubes with unit of each axis: m.}
			\label{fig 13}
		\end{figure}
		
		The frequency bands below 1 GHz are heavily occupied whereas the frequency bands above 1 GHz with the distance between BSs ranging from 400 m to 1500 m are mostly vacant \cite{ref13}. Therefore, assuming that the measured frequency band is above 1 GHz, and the distance between BSs is about 1000 m in this paper. In addition, the shape of licensed networks' coverage is a sphere, and the location coordinates of the BSs related to licensed network 1, 2 and 3 are set as (0 m, 0 m, 0 m) with radius 700 m, (0 m, 1000 m, 0 m) with radius 600 m, and (1000 m, 1000 m, 0 m) with radius 800 m, respectively. The 3D space is first divided into small cubes and the UAV is applied to measure the radio parameters of the cubes according to the proposed 3D SOM algorithm, which measures the center of a cube once and regards the measurement result as the radio parameter of the cube. Then, we apply Monte Carlo simulation method as well as \eqref{equ 35} to obtain the error of 3D SOM. Finally, the number of measurements and the flight distance of UAV are obtained by finding the shortest measurement path of the cubes using the proposed 3D SOM algorithm. In Figs. \ref{fig 9}-\ref{fig 22}, the numerical results are obtained by measuring the center of a cube using UAV.
		
		Fig. \ref{fig 9}, Fig. \ref{fig 10} and Fig. \ref{fig 11} show the measurement process of 3D SOM algorithm with three licensed networks. We take the number of cubes $M=9^3$ and the interval $d = 4$ as an example, such that three iterations are carried out to complete the entire 3D SOM process. The cubes to be measured for each iteration are marked with circles in Fig. \ref{fig 9a}, Fig. \ref{fig 10a} and Fig. \ref{fig 11a}. The flight path of UAV for SOM using the ACO algorithm for each iteration are marked with black lines in Fig. \ref{fig 9b}, Fig. \ref{fig 10b} and Fig. \ref{fig 11b}. It is observed that when $r=1$, the flight path of UAV is regular. However, with the increasing of $r$, the cubes to be measured are more and more close to boundary surface of licensed networks. Thus the area of the boundary surface of licensed networks has an impact on the number of cubes to be measured and will further impact the length of the flight path of UAV.
		
	    Fig. \ref{fig 12} and Fig. \ref{fig 13} show the results of 3D SOM with $M=9^3$ and $M={17}^3$ cubes using the proposed 3D SOM algorithm, respectively. The results imply that the accuracy of SOM increases with the increase of the number of cubes $M$. However, the overhead of the proposed 3D SOM algorithm will also increase with the increase of $M$.

		\begin{figure}[htbp]
			%\centering
			\includegraphics[width=0.20\textwidth]{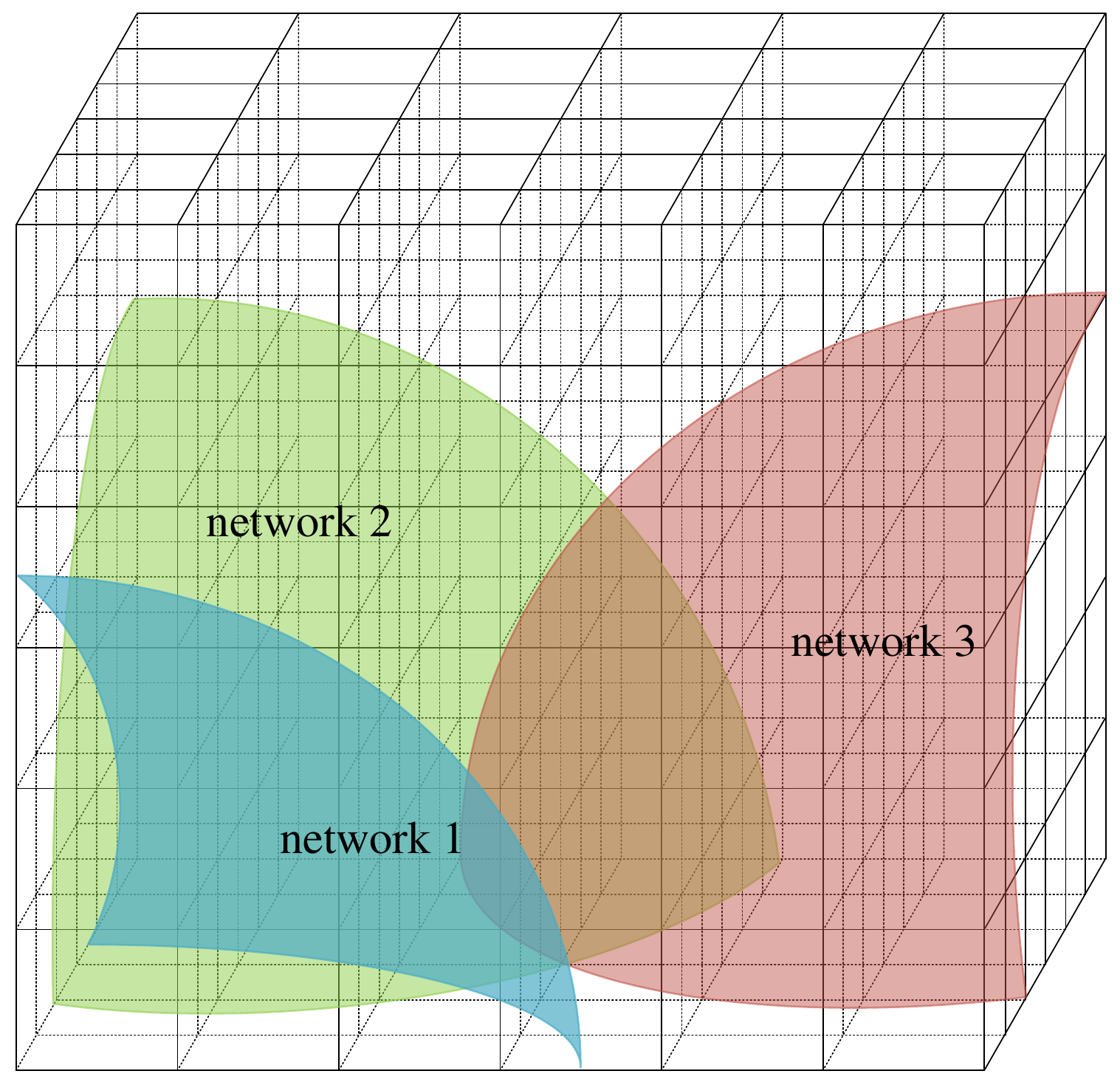}
			\includegraphics[width=0.25\textwidth]{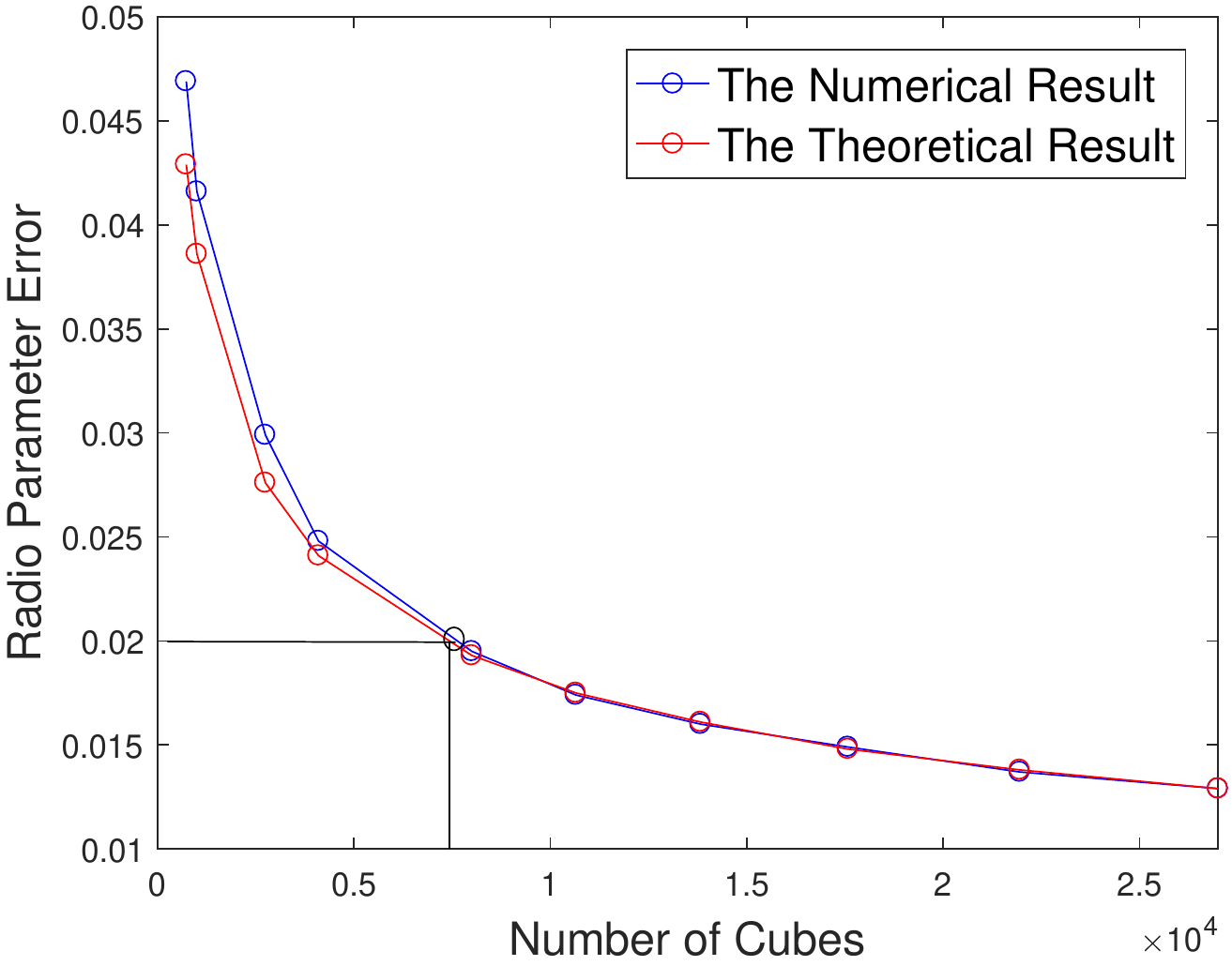}
			\caption{The number of cubes vs. the RPE with three licensed networks.}
			\label{fig 14}
		\end{figure}
		\begin{figure}[htbp]
			%\centering
			\includegraphics[width=0.20\textwidth]{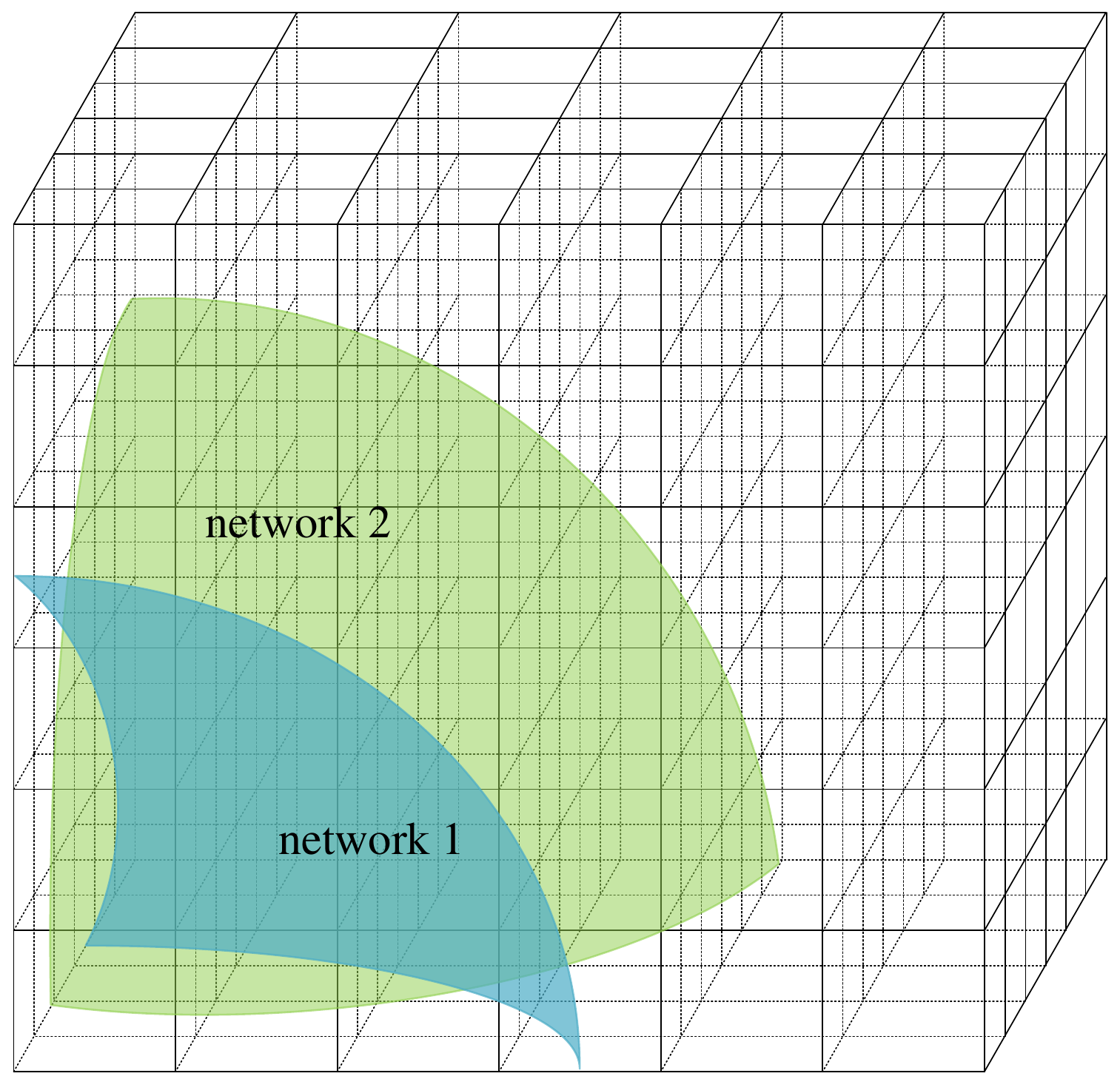}
			\includegraphics[width=0.26\textwidth]{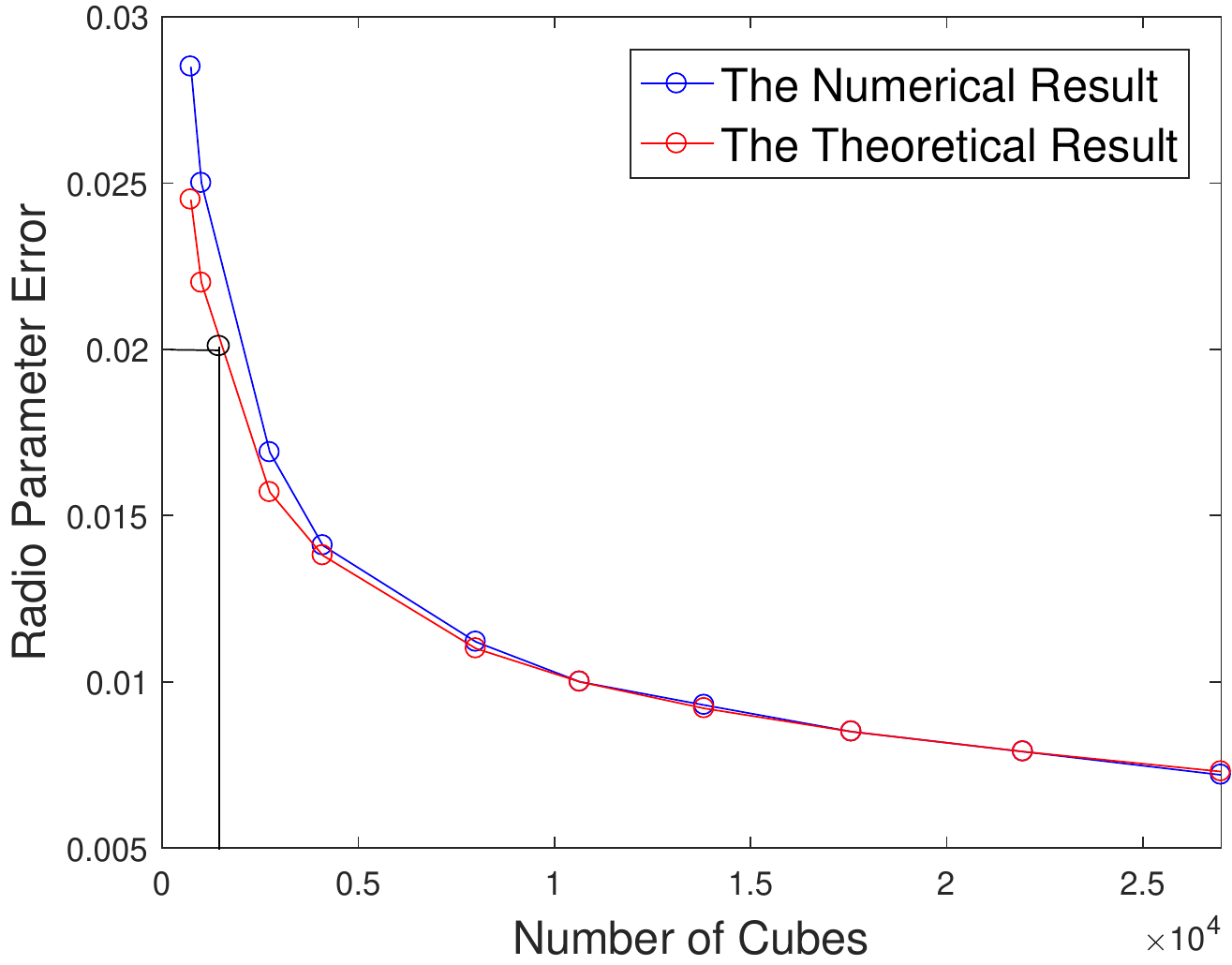}
			\caption{The number of cubes vs. the RPE with two licensed networks.}
			\label{fig 15}
		\end{figure}
		
		The relation between the RPE and the number of cubes is shown in Fig. \ref{fig 14} and Fig. \ref{fig 15} for three and two licensed networks, respectively. With the surface of licensed networks being the sphere, the distribution of the random variables $X = x$, $\Theta = \theta$ and ${\rm A} = \alpha$ can be approximated using uniform distribution. These figures verify the correction of the theoretical results in \textbf{Theorem 2} since the numerical results in Fig. \ref{fig 14} and Fig. \ref{fig 15} are close to the theoretical result calculated using \eqref{equ 31} in \textbf{Theorem 2}. As the number of cubes $M$ increases, the RPE decreases. To achieve the same RPE, the value of $M$ in Fig. \ref{fig 14} is larger than that in Fig. \ref{fig 15}. The essential reason is revealed in \textbf{Theorem 2}, which proves that the RPE is increasing with the increase of the area of the boundary surface of licensed networks $S$. Since the $S$ with three licensed networks is larger than that with two licensed networks, the RPE in Fig. \ref{fig 14} is larger than that in Fig. \ref{fig 15} with the same $M$.
		
		\begin{figure}[t]
			\centering
			\includegraphics[width=0.34\textheight]{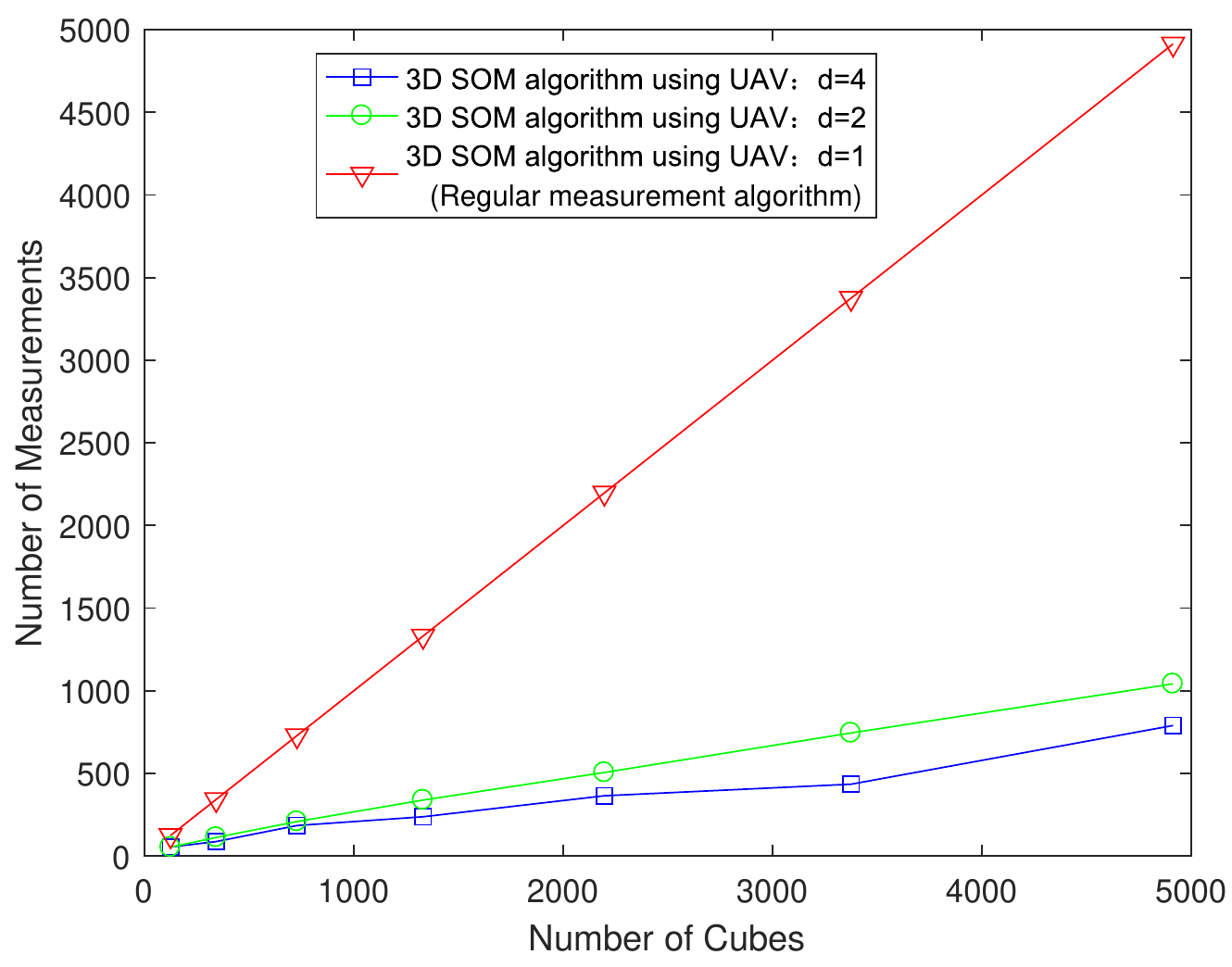}
			\DeclareGraphicsExtensions.
			\caption{The number of measurements vs. the number of cubes.}
			\label{fig 21}
		\end{figure}
		\begin{figure}[t]
			\centering
			\includegraphics[width=0.34\textheight]{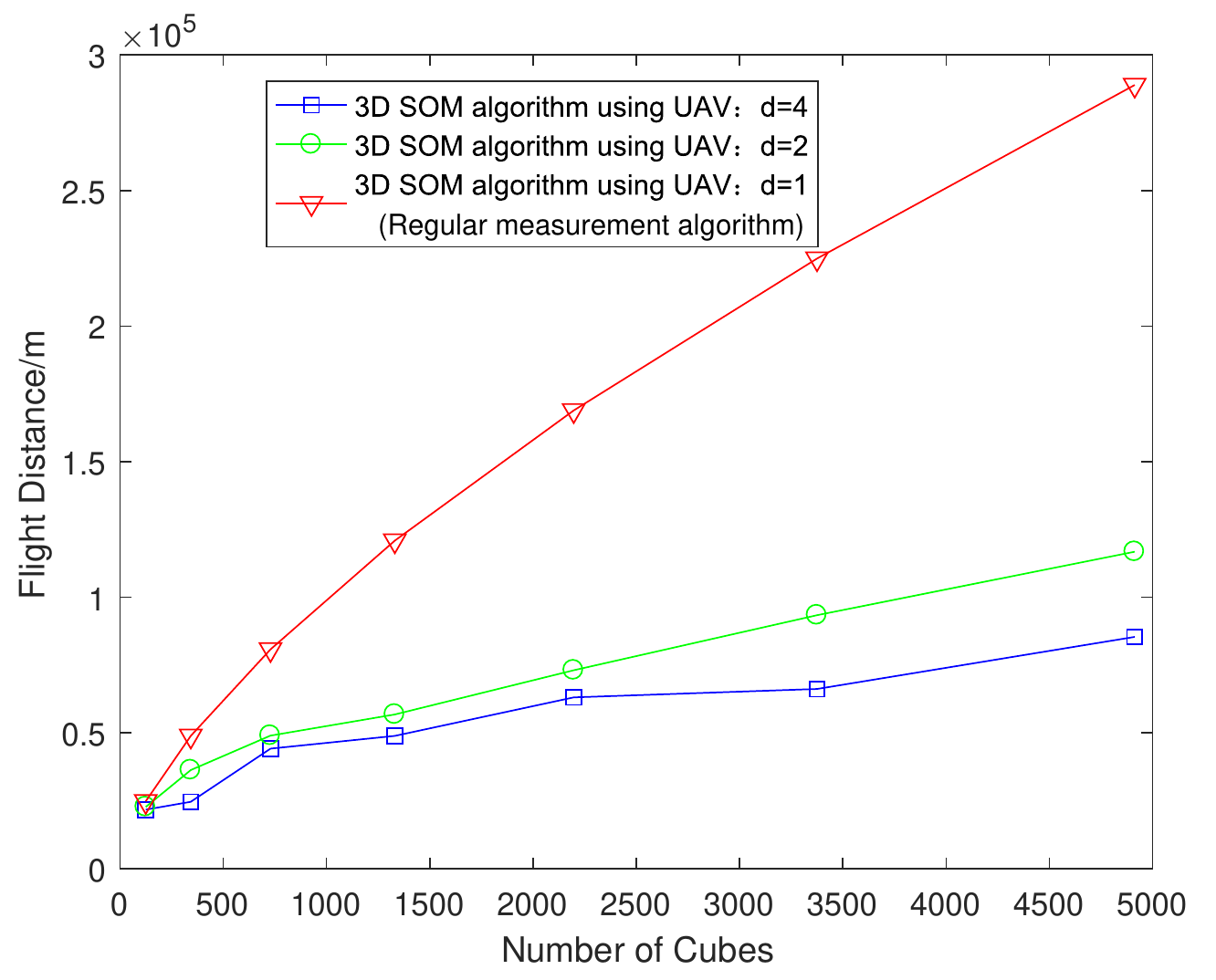}
			\DeclareGraphicsExtensions.
			\caption{The flight distance vs. the number of cubes.}
			\label{fig 22}
		\end{figure}
		\begin{figure}[t]
			\centering
			\includegraphics[width=0.34\textheight]{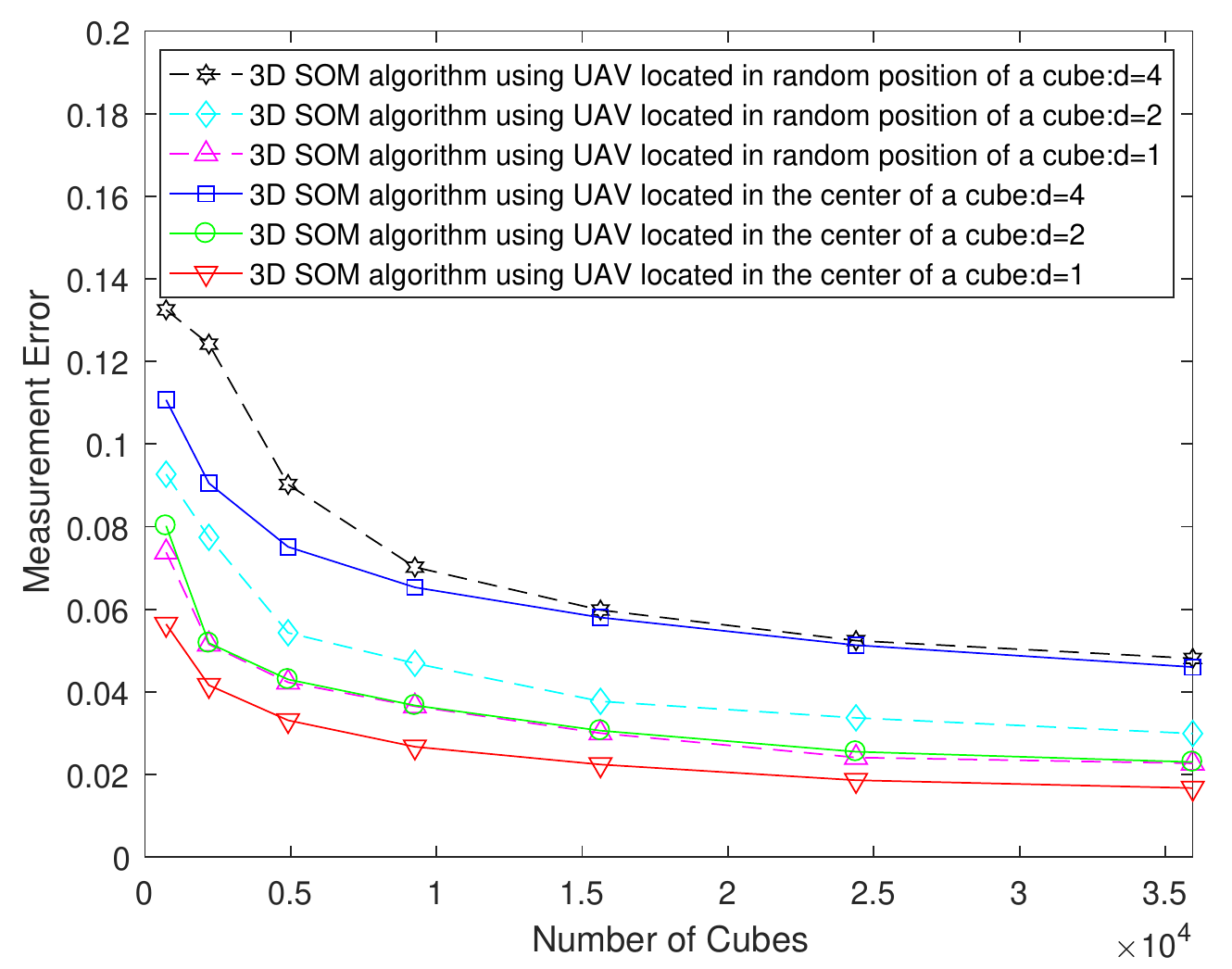}
			\DeclareGraphicsExtensions.
			\caption{The measurement error vs. the number of cubes.}
			\label{fig 23}
		\end{figure}
		
		Fig. \ref{fig 21} shows the relation between the number of measurements and the number of cubes, which reveals that the proposed 3D SOM algorithm can reduce the number of measurements compared with regular measurement algorithm with interval $d=1$. Besides, the number of measurements decreases with increase of the interval $d$, which verifies \textbf{Theorem 3}.
		
		Fig. \ref{fig 22} shows the relation between the flight distance and the number of cubes, which reveals that the 3D SOM algorithm can shorten the flight distance of the UAV compared with regular measurement algorithm with interval $d=1$, and the flight distance decreases with increase of the interval $d$. Therefore, the proposed 3D SOM algorithm reduces the flight time and measurement time as shown in Fig. \ref{fig 21} and Fig. \ref{fig 22}, namely, improves the measurement efficiency, which is expressed as the time including the measurement time and flight time required to obtain the radio parameter of every cube in the entire 3D space. Moreover, the proposed 3D SOM algorithm reduces the energy consumption in UAV-assisted SOM compared with regular measurement algorithm with interval $d=1$. However, the measurement error increases with increase of the interval $d$, which is shown in Fig. \ref{fig 23}. The measurement error is decreasing with the increase of the number of cubes. Hence, there exists a tradeoff between accuracy and efficiency in 3D SOM. Besides, the locations of UAV will also affect the measurement error. As shown in Fig. \ref{fig 23}, the 3D SOM algorithm using the measurement located in the center of a cube can reduce the measurement error compared with the 3D SOM algorithms using random measurement position within a cube.

		\section{Conclusion}
		
		In this paper, we study 3D SOM using UAV. The 3D space is first divided into small cubes and then the performance of 3D SOM is analyzed to study the tradeoff between accuracy and efficiency. Moreover, a fast 3D SOM algorithm is designed combined with UAV path planning. The performance of the proposed 3D SOM algorithm is analyzed. Simulation results have demonstrated the measurement process and results of 3D SOM algorithm, which improves the measurement efficiency by requiring less measurement time and flight time of UAV for satisfactory performance. Besides, we discover that the error of 3D SOM depends on the measurement position of the UAV and the 3D SOM algorithm using the measurement located in the center of a small cube achieves small error. The 3D SOM using cooperative multiple UAVs could obtain higher measurement efficiency, which is the future work of 3D SOM using UAVs.


\begin{thebibliography}{1}
			
			\bibitem{ref1}
			M. Latva-aho, K. Lepp\"anen, F. Clazzer, and A. Munari, ``Key Drivers and Research Challenges for 6G Ubiquitous Wireless Intelligence (white paper),'' \textit{6G Flagship}, University of Oulu, Finland, Sep. 2019, pp. 1-33.
			
			\bibitem{ref2}
			G. G\"ur, ``Expansive Networks: Exploiting Spectrum Sharing for Capacity Boost and 6G Vision,'' \emph{J. Commun. Netw.}, vol. 22, no. 6, pp. 444-454, Dec. 2020.
			
			\bibitem{ref3}
			M. Matinmikko-Blue, S. Yrj\"ol\"a, and P. Ahokangas, ``Spectrum management in the 6G era: The role of regulation and spectrum sharing,'' in \textit{Proc. 2nd 6G Wireless Summit (6G SUMMIT)}, Mar. 2020, pp. 1-5.
			
			\bibitem{ref4}
			F. Paisana, Z. Khan , J. Lehtom\"aki , L. A. DaSilva, and R. Vuohtoniemi, ``Exploring Radio Environment Map Architectures for Spectrum Sharing in the Radar Bands,'' in \textit{Proc. 23rd Int. Conf. Telecommun. (ICT)}, May 2016, pp. 1-6.
			
			\bibitem{ref5}
			Y. Zhao, L. Morales, J. Gaeddert, K. K. Bae, J.-S. Um, and J. H. Reed, ``Applying radio environment maps to cognitive wireless regional area networks,'' in \textit {Proc. 2nd IEEE Int. Symp. New Frontiers Dyn. Spectr. Access Netw.},  Apr. 2007, pp. 115-118.
			
			\bibitem{ref6}
			I. Kakalou, K. Psannis, S. K. Goudos, T. V. Yioultsis, N. V. Kantartzis and Y. Ishibashi, ``Radio Environment Maps for 5G Cognitive Radio Network,'' in \textit{Proc. 8th Int. Conf. Modern Circuits Syst. Technol. (MOCAST)}, May 2019, pp. 1-4.
			
			\bibitem{ref7}
			J. Perez-Romero \textit{et al.}, ``On the use of radio environment maps for interference management in heterogeneous networks,'' \textit{IEEE Communications Mag.}, vol. 53, no. 8, pp. 184-191, Aug. 2015.
			
			\bibitem{ref8}
			Z. Zhang \textit{et al.} ``6G wireless networks: Vision, requirements, architecture, and key technologies,'' \textit{IEEE\;Veh.\;Technol.\;Mag.}, vol. 14, no. 3, pp. 28-41, Sep. 2019.
			
			\bibitem{ref9}
			M. H\"oyhty\"a \textit{et al.} ``Spectrum occupancy measurements: A survey and use of interference maps,'' \textit{IEEE Commun. Surveys Tuts.}, vol. 18, no. 4, pp. 2386-2414, 4th Quart., 2016.
			
			\bibitem{ref10}
			Y. Chen and H. -S. Oh, ``A Survey of Measurement-Based Spectrum Occupancy Modeling for Cognitive Radios,'' in \textit{IEEE Commun. Surveys Tuts.}, vol. 18, no. 1, pp. 848-859, 1st Quart., 2016.
			
			\bibitem{ref11}
			Z. Sun, G. J. Bradford, and J. N. Laneman, ``Sequence detection algorithms for PHY -layer sensing in dynamic spectrum access networks,'' \textit{IEEE J. Sel. Topics Signal Process.}, vol. 5, no. 1, pp. 97-109, Feb. 2011.
			
			\bibitem{ref12}
			M. Lopez-Benitez and F. Casadevall, ``Statistical prediction of spectrum occupancy perception in dynamic spectrum access networks,'' in \textit{Proc. IEEE Int. Conf. Commun. (ICC)}, Jun. 2011, pp. 1-6.
			
			\bibitem{ref13}
			H. Mosavat-Jahromi, Y. Li, L. Cai and J. Pan, ``Prediction and Modeling of Spectrum Occupancy for Dynamic Spectrum Access Systems,'' \textit{IEEE Trans. Cognit. Commun. Netw.}, vol. 7, no. 3, pp. 715-728, Sept. 2021.
			
			\bibitem{ref14}
			M. Matinmikko, M. Mustonen, M. H\"oyhty\"a, T. Rauma, H. Sarvanko, and A. M\"ammel\"a , ``Cooperative Spectrum Occupancy Measurements in the 2.4 GHz ISM Band,'' in \textit{Proc. 3rd Int. Symp. Appl. Sci. Biomed. Commun. Technol. (ISABEL)}, Nov. 2010, pp. 1-5.
			
			\bibitem{ref15}
			C. Wijenayake, A. Madanayake, J. Kota, L. Bruton, ``Space-time spectral white spaces in cognitive radio: theory, algorithms, and circuits,'' \textit{IEEE\;J.\;Emerg.\;Sel.\;Top.\; Circuits\;Syst.}, vol. 3, no. 4, pp. 640-653, Dec. 2013.
			
			\bibitem{ref16}
			A. Ivanov \textit{et al.}, ``3D Interference Mapping for Indoor IoT Scenarios,'' in \textit{Proc. 43rd Int. Conf. Telecommun. Signal Process. (TSP)}, Jul. 2020, pp. 265-269.
			
			\bibitem{ref17}
			M. A. Ayg\"ul \textit{et al.}, ``Spectrum occupancy prediction exploiting time and frequency correlations through 2D-LSTM,'' in \textit{Proc. IEEE 91st Veh. Technol. Conf. (VTC-Spring)}, May 2020, pp. 1-5.
			
			\bibitem{ref18}
		    G. Ding, Q. Wu, L. Zhang, Y. Lin, T. A. Tsiftsis, and Y.-D. Yao, ``An Amateur Drone Surveillance System Based on the Cognitive Internet of Things,'' \textit{IEEE\;Commun.\;Mag.}, vol. 56, no. 1, pp. 29-35, Jan. 2018.
			
			\bibitem{ref19}
			A. Y. Umeyama, J. L. Salazar-Cerreno, and C. J. Fulton, ``UAV-Based Far-Field Antenna Pattern Measurement Method for Polarimetric Weather Radars: Simulation and Error Analysis,'' \textit{IEEE Access}, vol. 8,  pp. 191124-191137, 2020.
			
			\bibitem{ref20}
			Z. Niu, X. S. Shen, Q. Zhang, and Y. Tang, ``Space-air-ground integrated vehicular network for connected and automated vehicles: Challenges and solutions,'' \textit{Intell. Converged Netw.} , vol. 1, no. 2, pp. 142-169, Sep. 2020.
			
			\bibitem{ref21}
			A. Al-Hourani, ``Interference Modeling in Low-Altitude Unmanned Aerial Vehicles ,'' \textit{IEEE Wireless Commun. Lett.}, vol. 9, no. 11, pp. 1952?1955, Nov. 2020.
			
			\bibitem{ref22}
			S. Faint, X. O. \"Ureten, and T. Willink,  ``Impact of the number of sensors on the network cost and accuracy of the radio environment map,'' in \textit{Proc. CCECE}, May 2010, pp. 1-5.
			
			\bibitem{ref23}
			Z. Wei, Q. Zhang, Z. Feng, W. Li, and T. A. Gulliver, ``On the construction of Radio Environment Maps for Cognitive Radio Networks, '' in \textit{Proc. IEEE Wireless Commun. Netw. Conf. (WCNC)}, Apr. 2013, pp. 4504-4509.
			
			\bibitem{ref24}
			M. Dorigo and C. Blum, ``Ant colony optimization theory: A survey,'' \textit{Theor. Comput. Sci.}, vol. 344, nos. 2-3, pp. 243-278, Aug. 2015.
			
			\bibitem{ref25}
			T. Zaza and A. Richards, ``Ant colony optimization for routing and tasking problems for teams of UAVS,'' in \textit{Proc. UKACC Int. Conf. Control (CONTROL)}, 2014, pp. 652-655.
			
			\bibitem{ref26}
			H. Liang, W. Gao, J. H. Nguyen, M. F. Orpilla, and W. Yu, ``Internet of Things Data Collection Using Unmanned Aerial Vehicles in Infrastructure Free Environments,''\textit{ IEEE Access}, vol. 8, pp. 3932-3944, 2020.			
		\end{thebibliography}
\end{document}